\documentclass{article}

\usepackage[square,numbers]{natbib}

\usepackage[final]{neurips_2025}

\usepackage[utf8]{inputenc} % allow utf-8 input
\usepackage[T1]{fontenc}    % use 8-bit T1 fonts
\usepackage{hyperref}       % hyperlinks
\usepackage{url}            % simple URL typesetting
\usepackage{booktabs}       % professional-quality tables
\usepackage{amsfonts}       % blackboard math symbols
\usepackage{nicefrac}       % compact symbols for 1/2, etc.
\usepackage{microtype}      % microtypography
\usepackage{xcolor}         % colors
\usepackage[table]{xcolor}

\definecolor{mine}{RGB}{205, 232, 248}
\usepackage{graphicx}
\usepackage{booktabs}
\usepackage{multirow}
\usepackage{array} 
\usepackage{algorithm}
\usepackage{algpseudocode}
\usepackage{amsmath,amssymb,amsthm}
\usepackage{booktabs}
\usepackage{caption}
\usepackage{subcaption}
\usepackage{wrapfig}
\usepackage{enumitem}
\definecolor{mine}{RGB}{205, 232, 248}
\newtheorem{proposition}{Proposition}
\newtheorem{lemma}{Lemma}
\usepackage{wrapfig}
\usepackage{minted}
\definecolor{gray}{rgb}{0.5,0.5,0.5}

\usepackage{listings}
\usepackage{xcolor}

\definecolor{codegreen}{rgb}{0,0.6,0}
\definecolor{codegray}{rgb}{0.5,0.5,0.5}
\definecolor{codepurple}{rgb}{0.58,0,0.82}
\definecolor{backcolour}{rgb}{0.95,0.95,0.92}

\lstdefinestyle{mystyle}{
    backgroundcolor=\color{backcolour},   
    commentstyle=\color{codegreen},
    keywordstyle=\color{magenta},
    numberstyle=\tiny\color{codegray},
    stringstyle=\color{codepurple},
    basicstyle=\ttfamily\footnotesize,
    breakatwhitespace=false,         
    breaklines=true,                 
    captionpos=b,                    
    keepspaces=true,                 
    numbers=left,                    
    numbersep=5pt,                  
    showspaces=false,                
    showstringspaces=false,
    showtabs=false,                  
    tabsize=2
}

\title{Towards Robust Zero-Shot Reinforcement Learning}

\author{%
  Kexin Zheng$^{1}$\footnotemark[1]\enspace\footnotemark[3]~~, Lauriane Teyssier$^{2}$\footnotemark[1]~~, Yinan Zheng$^{2}$, Yu Luo$^{3}$, Xianyuan Zhan$^{2,4}$\footnotemark[2]\\
$^1$ The Chinese University of Hong Kong \quad$^2$ Tsinghua University \\
$^3$ Huawei Noah's Ark Lab \quad$^4$ Shanghai Artificial Intelligence Laboratory \\
\texttt{1155173723@link.cuhk.edu.hk,} \texttt{zhanxianyuan@air.tsinghua.edu.cn}\\
}

\begin{document}

\maketitle
\renewcommand{\thefootnote}{\fnsymbol{footnote}}
\footnotetext[1]{Equal contribution.}
\footnotetext[2]{Corresponding author.}
\footnotetext[3]{Work done during internships at Institute for AI Industry Research (AIR), Tsinghua University.}

\begin{abstract}
The recent development of zero-shot reinforcement learning (RL) has opened a new avenue for learning pre-trained generalist policies that can adapt to arbitrary new tasks in a zero-shot manner. While the popular Forward-Backward representations (FB) and related methods have shown promise in zero-shot RL, we empirically found that their modeling lacks expressivity and that extrapolation errors caused by out-of-distribution (OOD) actions during offline learning sometimes lead to biased representations, ultimately resulting in suboptimal performance.
To address these issues, we propose \textit{\textbf{B}ehavior-\textbf{RE}gulariz\textbf{E}d \textbf{Z}ero-shot RL with \textbf{E}xpressivity enhancement} (BREEZE), an upgraded FB-based framework that simultaneously enhances learning stability, policy extraction capability, and representation learning quality. BREEZE introduces behavioral regularization in zero-shot RL policy learning, transforming policy optimization into a stable in-sample learning paradigm. Additionally, BREEZE extracts the policy using a task-conditioned diffusion model, enabling the generation of high-quality and multimodal action distributions in zero-shot RL settings. Moreover, BREEZE employs expressive attention-based architectures for representation modeling to capture the complex relationships between environmental dynamics. Extensive experiments on ExORL and D4RL Kitchen demonstrate that BREEZE achieves the best or near-the-best performance while exhibiting superior robustness compared to prior offline zero-shot RL methods. The official implementation is available at:~\url{https://github.com/Whiterrrrr/BREEZE}.
\end{abstract}
\section{Introduction}

Reinforcement learning (RL) has become a cornerstone of artificial intelligence, enabling transformative advances in robotics~\citep{Tang2024DeepRL}, autonomous systems~\citep{rl4ad}, industrial control~\citep{zhan2022deepthermal}, and large language models (LLM)~\citep{wang2025reinforcementlearningenhancedllms}. However, its real-world adoption faces two persistent challenges: the reliance on human-provided reward functions and its task-specific learning paradigm, which limits adaptability to novel or multiple tasks. These challenges sparked growing interest in zero-shot RL~\citep{Touati2021FB,Touati2022DoesZR,park2024foundationpolicieshilbertrepresentations,Jeen2023ZeroShotRL,Borsa2018UniversalSF,amy2024FE,Tirinzoni2025ZeroShotWH,levine2024FRE}, which enables learning a versatile agent through pretraining on reward-free transitions and then zero-shot adaptation to arbitrary reward functions during inference. This opens up new possibilities for developing general-purpose RL agents capable of generalizing across diverse tasks in open-world scenarios.

Existing zero-shot RL approaches mainly fall into two categories: task/skill-conditioned RL~\citep{levine2024FRE,amy2024FE} and dynamic representation-based methods~\citep{Borsa2018UniversalSF,Barreto2016SuccessorFF, Touati2022DoesZR,Jeen2023ZeroShotRL,Tirinzoni2025ZeroShotWH,park2024foundationpolicieshilbertrepresentations}. The first category encodes demonstrations or reward functions into embeddings as conditioning signals for policy learning. This preserves the task generalization capability of policies, but often results in heavy, manually designed pretraining without optimality guarantees. The second category of methods adopts a more principled approach by decomposing the problem into dynamic representations that can be recomposed for novel tasks without retraining. Among dynamic representation-based methods, Forward-Backward representations (FB)~\citep{Touati2021FB,Touati2022DoesZR} have recently attracted notable attention, factorizing occupancy measures into two components: a forward representation that captures policy dynamics and a backward representation that encodes global state information. Through offline, unsupervised pretraining, the FB framework learns linearized representations that approximate value functions for arbitrary tasks, holding great promise for zero-shot generalization.

However, despite the elegant theoretical framework provided by the FB representations, our empirical studies reveal that the successor measures learned through existing FB-based methods are often inconsistent and biased, thereby compromising the stability and overall performance (see Section~\ref{sec:motivation} for details). In this paper, we identify two causes for the shortcomings of existing FB-based methods. First, learning complex approximators and multimodal behaviors requires highly expressive models, which current FB methods lack in both their representations and policy. Second, the offline, unsupervised pretraining stage suffers from extrapolation error due to out-of-distribution (OOD) actions, which is a similar problem to offline RL~\citep{Fujimoto2018OffPolicyDR,kumar2019stabilizing,kummar2020cql}, but exhibits more complex behavior. As our results show, naively integrating value constraints can be ineffective~\citep{Jeen2023ZeroShotRL}, indicating a need for a more delicate OOD regularization mechanism.

Based on the above observations, we propose \textit{\textbf{B}ehavior-\textbf{RE}gulariz\textbf{E}d \textbf{Z}ero-shot RL with \textbf{E}xpressivity enhancement} (BREEZE), a novel FB-based framework that simultaneously enhances offline learning stability and zero-shot generalization capability. First, we introduce a behavior-regularized reformulation of FB, which mitigates extrapolation errors while preserving the fidelity of representations. Second, we extract the policy using a task-conditioned diffusion model, enabling high-quality multimodal action distributions in zero-shot RL settings. Finally, we employ expressive representation networks based on the attention architecture to capture complicated dynamics. We conducted extensive experiments on the ExORL benchmark~\citep{yarats2022dontchangealgorithmchange} and the D4RL Kitchen dataset~\citep{fu2021d4rldatasetsdeepdatadriven}, under both full datasets and small-sample data regimes. The results demonstrate that BREEZE achieves the best or near-the-best performance while exhibiting superior robustness.
\section{Preliminary}
\paragraph{Reward-free Markov decision process (MDP).}
Zero-shot RL is typically formulated as a reward-free Markov Decision Process (MDP)~\citep{mdp}, defined by a tuple \( \mathcal{M} = (\mathcal{S}, \mathcal{A}, \mathcal{P}, \gamma) \), consisting of a state space $\mathcal{S}$, an action space $\mathcal{A}$, a transition kernel $\mathcal{P}:\mathcal{S} \times \mathcal{A} \to \Delta(\mathcal{S})$, and a discount factor $\gamma \in (0,1)$. Given an initial state-action pair \((s_0, a_0) \in \mathcal{S} \times \mathcal{A}\) and a policy \(\pi: \mathcal{S} \to \Delta(\mathcal{A})\), we denote $\text{Pr}(\cdot|s_0, a_0, \pi)$ a probability and $\mathbb{E}[\cdot|s_0,a_0,\pi]$ an expectation under the state-action trajectories \((s_0, a_0,...,s_t, a_t)_{t \geq 0}\) generated by sampling \(s_t \sim \mathcal{P}(\cdot | s_{t-1}, a_{t-1})\) and \(a_t\sim\pi(\cdot| s_t)\). The state transition under \(\pi\) is given by \(\mathcal{P}_{\pi}(\mathrm{d}s' | s) = \int \mathcal{P}(\mathrm{d}s' | s, a)\pi(\mathrm{d}a | s)\).

Approximate dynamic programming-based RL~\citep{powell2007approximate} uses an action-value \(Q\)-function, or optionally a state-value \(V\)-function. The \(Q\)-function and \(V\)-function for \(\pi\) starting at \(s_0,a_0\) under a given reward function \(r: \mathcal{S} \to \mathbb{R}\) are respectively defined as \(Q^{\pi}_r(s_0, a_0):= \sum_{t \geq 0} \gamma^t \mathbb{E}[r(s_{t+1}) | s_0, a_0, \pi]\) and \(V^{\pi}_r(s_0):= \sum_{t \geq 0} \gamma^t \mathbb{E}[r(s_{t+1}) | s_0, \pi]\). The goal of the zero-shot RL problem is to train a task-agnostic policy given an offline dataset \( \mathcal{D}=\{s_i, a_i, s_{i+1}\}_{i=1}^N\) generated by an unknown behavior policy $\mu(\cdot|s)$, that can later generalize to any downstream task $z$ defined by a reward function \(r_{\text{eval}}: \mathcal{S} \to \mathbb{R} \), i.e. find \(\max_{\pi} \ \mathbb{E} [\sum_{t=0}^{\infty} \gamma^{t} r_{\text{eval}}\left(s_{t+1}\right) | s_{0}, a_{0}, \pi]\),
relying only on a small set of reward-labeled samples with no further finetuning.

\paragraph{Forward-Backward representations (FB).} FB introduces a rank-$d$ approximation of the successor measure, which is defined as the expected discounted occurrences of future states $s_+\in \mathcal{S}_+$ after starting from $(s_0, a_0)$ under a policy $\pi$:
\begin{equation}
\label{eq:M}
    M^{\pi}(s_0, a_0, \mathcal{S}_+) := \sum_{t \geq 0} \gamma^t \, \text{Pr}\big(s_{t+1} \in \mathcal{S}_+ \mid s_0, a_0, \pi \big) \, \quad \forall \mathcal{S}_+ \subset \mathcal{S}.
\end{equation}
Following this definition, $Q$-function under policy $\pi$ can be expressed as: 
\begin{equation}
Q_r^{\pi}(s, a) = \int_{s_+ \in \mathcal{S}} M^{\pi}(s, a, ds_+) \, r(s_+). 
\end{equation}
Given a representation space \(\mathbb{R}^d\), a state distribution $\rho$, a task vector $z \in \mathbb{R}^d$ and a policy $\pi_z$ parameterized by $z$, FB expresses $M^{\pi_z}$ as the product of a forward representation $F: \mathcal{S} \times \mathcal{A} \times \mathbb{R}^d \to \mathbb{R}^d$ and a backward representation $B: \mathcal{S} \to \mathbb{R}^d$, resulting in an optimal policy $\pi_z$:
\begin{equation}
    \label{eq:fb}
    M^{\pi_z}(s_0,a_0,ds_+)\approx F(s_0,a_0,z)^\top B(s_+)\rho(ds_+),\quad \pi_z(s):=\arg\max_a F(s,a,z)^\top z,
\end{equation}
where $z$ is defined as $z := \mathbb{E}_{s \sim \rho}[r(s)B(s)]$ from a few samples with a known reward function $r$. The corresponding $Q$-function can be obtained immediately as $Q_z(s, a) =  F(s, a, z)^\top z$. 

Since the successor measure $M^{\pi_z}$ satisfies a Bellman-like equation $M^{\pi_z}=\mathcal{P}+\gamma \mathcal{P}_{\pi_z}M^{\pi_z}$, FB derives the temporal difference (TD) objective~\citep{suttonTD} on $M^{\pi_z}$ as:
\begin{equation}
\label{fb_loss}
\begin{gathered}
\mathcal{L}_{\text{FB}} = \mathbb{E}_{\substack{(s_t, a_t, s_{t+1}) \sim D \\ s_+ \sim D}} \left[ \left( F(s_t, a_t, z)^\top B(s_+) - \gamma \bar{F}(s_{t+1}, \pi_z(s_{t+1}), z)^\top \bar{B}(s_+) \right)^2 \right] \\
- 2 \, \mathbb{E}_{(s_t, a_t, s_{t+1}) \sim D} \left[ F(s_t, a_t, z)^\top B(s_{t+1}) \right],
\end{gathered}
\end{equation}
and incorporate the following objective on $F$ to enforce the Bellman property on the $Q$-function:
\begin{equation}
\label{f_loss}
\begin{gathered}
\scalebox{0.95}{$
\mathcal{L}_{\text{F}} = \mathbb{E}_{\substack{(s_t, a_t, s_{t+1}) \sim D}} \left[ \left( F(s_t, a_t, z)^\top z - B(s_{t+1})^\top \mathbb{E}_{D}[BB^\top]^{-1}z -  \gamma \bar{F}(s_{t+1}, \pi_z(s_{t+1}), z)^\top z \right)^2 \right]
$}
\end{gathered}
\end{equation}
where $B(s_{t+1})^\top \mathbb{E}_{D}[BB^\top]^{-1}z$ is the implicit reward estimation and $D$ denote the data distribution.

Recent extensions of FB incorporate additional regularization to improve learning performance. For example, \citet{Jeen2023ZeroShotRL} noted that the use of $\pi_z(s_{t+1})$ in Eq.~(\ref{fb_loss}) during offline learning could introduce OOD extrapolation errors. To resolve this problem, they introduce additional CQL-style~\citep{kummar2020cql} regularizers on $M^{\pi_z}$ (MCFB) or $Q_z$ (VCFB) in $\mathcal{L}_{\text{FB}}$ for OOD regularization (i.e., decrease values for OOD actions and increase those for data samples).
\section{Pitfalls of Exiting FB-Based Methods} \label{sec:motivation}

\begin{figure}[b]
  \centering
  \subcaptionbox{\footnotesize Distribution of empirically evaluated $M^{\pi_z}$}{
    \includegraphics[width=0.233\linewidth]{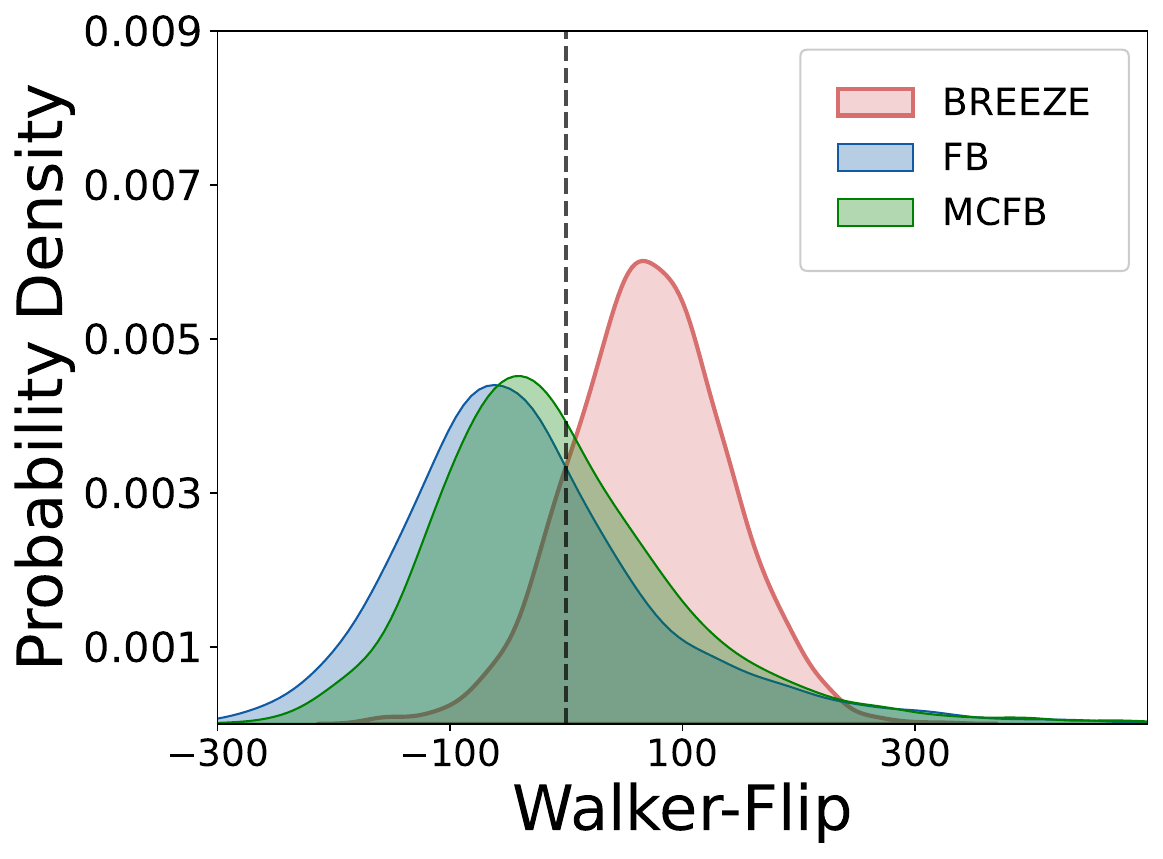}
    \includegraphics[width=0.233\linewidth]{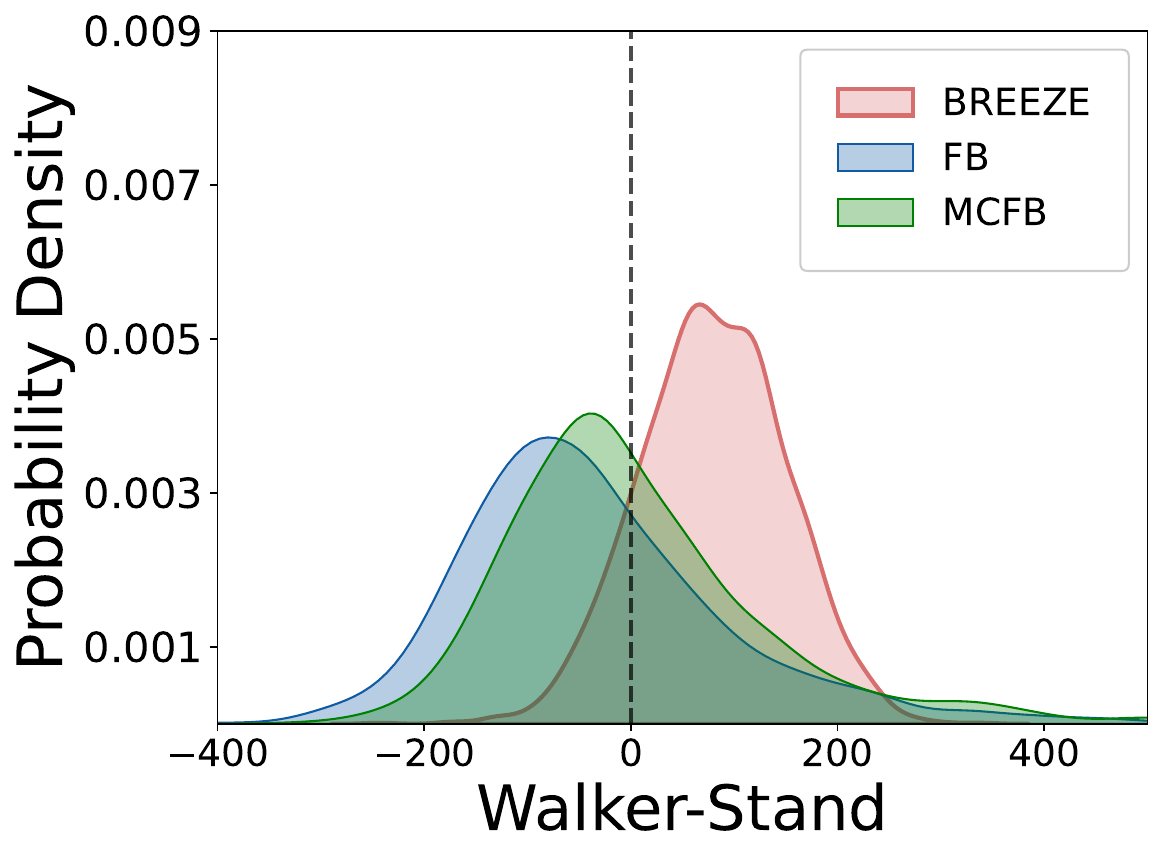}
  }%
  \hspace{0.02\linewidth}
  \subcaptionbox{\footnotesize Distribution of empirically evaluated $Q_z$}{
    \includegraphics[width=0.233\linewidth]{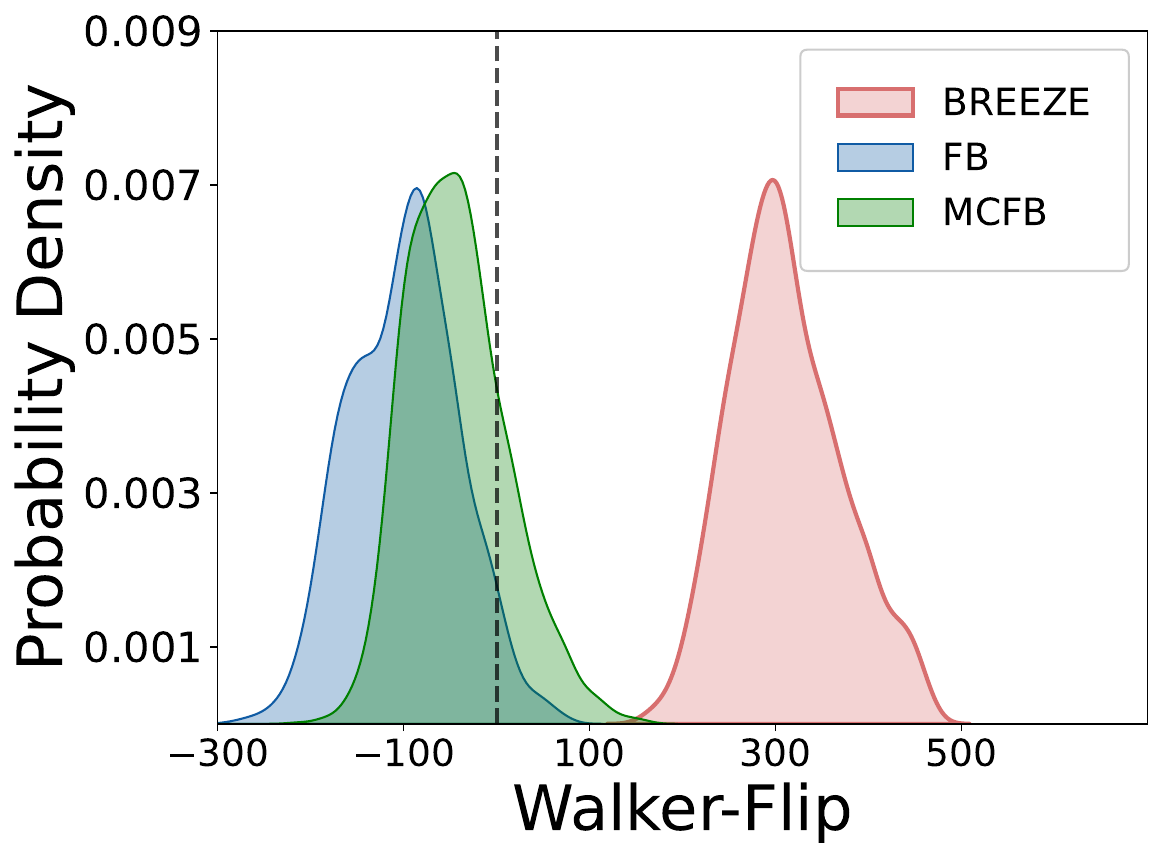}
    \includegraphics[width=0.233\linewidth]{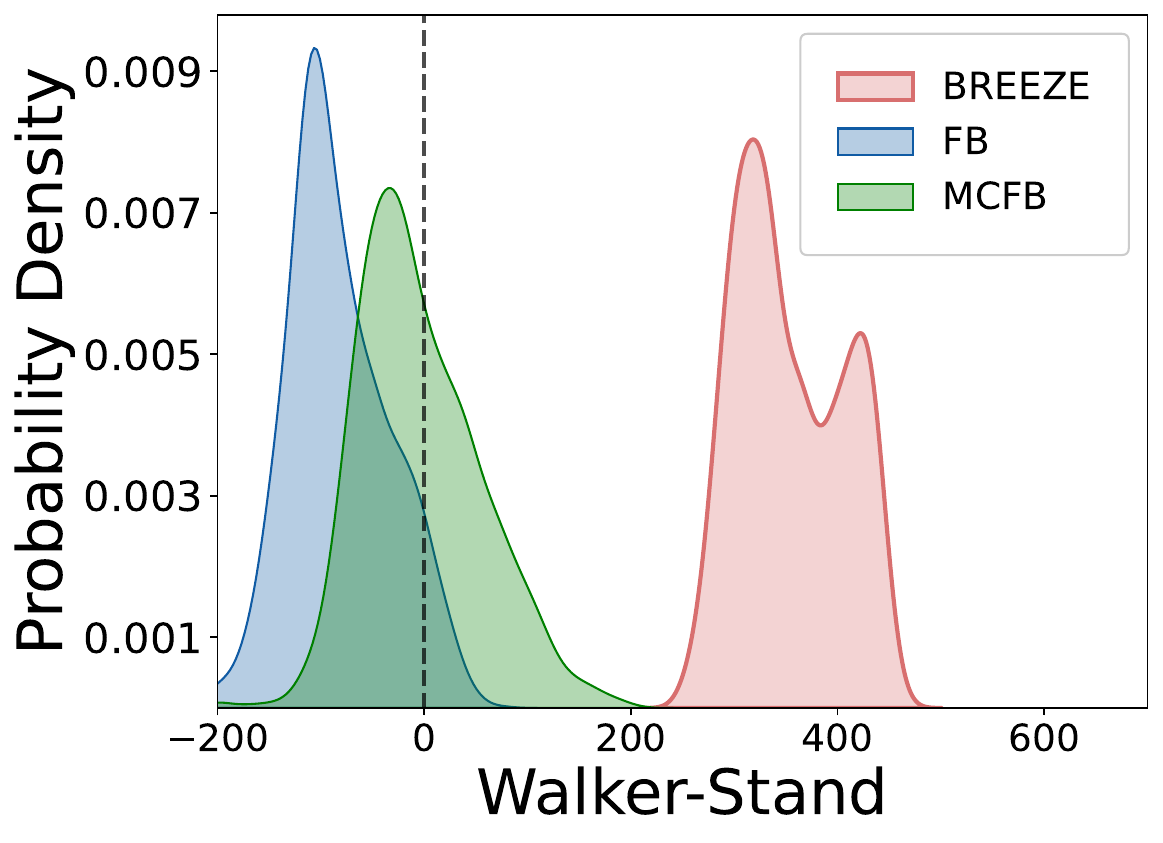}
  }
  \caption{\small Visualization of the empirical $M^{\pi_z}$ and $Q_z$ distributions during evaluation stage. We conduct experiments on FB-based methods on two Walker tasks using ExORL~\citep{yarats2022dontchangealgorithmchange} RND~\citep{Burda2018ExplorationBR} dataset. We use the learned $F$, $B$ representations to recover the task vector $z$ and compute $M^{\pi_z}=F^\top B$ and $Q_z=F^\top z$. Both the \textit{vanilla FB}~\cite{Touati2022DoesZR} and \textit{MCFB}~\citep{Jeen2023ZeroShotRL} result in a proportion of error scaling values.
  }
  \label{fig:M_V_distribution}
\end{figure}

FB offers an elegant theoretical foundation for zero-shot RL; however, our empirical findings reveal that FB-based methods often produce inconsistent and biased representations in practice. To illustrate this, we evaluate the $M^{\pi_z}$- and $Q_z$-value distributions derived from the learned $F$ and $B$ representations of the vanilla FB~\cite{Touati2022DoesZR}, the more recent MCFB~\cite{Jeen2023ZeroShotRL}, and our method on ExORL benchmark~\cite{yarats2022dontchangealgorithmchange} with RND~\citep{Burda2018ExplorationBR} datasets. For each method, we use its $B$ representation to derive the task vector $z$ under default settings with a fixed number of transitions. The $M^{\pi_z}$ values are computed from episode-start state-action pairs $(s_0,a_0) \sim \rho_0$ and randomly sampled batch of transitions $s_+ \sim \mathcal{D}$, while $Q$-values are evaluated using policy-induced actions $a \sim \pi_z(s)$. Results across additional environments and with method VCFB~\citep{Jeen2023ZeroShotRL} are provided in Appendix~\ref{app:empirical_motivation}.

Figure~\ref{fig:M_V_distribution} illustrates the value distributions from two tasks in the Walker domain. Although the successor measure $M^{\pi_z}$ is mathematically a positive quantity representing future state occupancy, the representations learned by existing FB-based methods fail to accurately capture this property. These distributions exhibit two notable discrepancies: a scale mismatch characterized by enormous absolute values, and the presence of a considerable number of invalid negative values. Such inaccuracies in the $F$ and $B$ representations subsequently affect downstream $Q$-value estimates, creating a consistent prediction bias across the system, which is further worsened by the suboptimal action rollout. The CQL-style regularizer in MCFB partially mitigates the issue, shifting the distributions of $M^{\pi_z}$ and $Q_z$ toward more plausible ranges. However, this correction is incomplete, and significant estimation bias persists, indicating the need for a more effective regularization strategy. In comparison, our method generates distributions for $M^{\pi_z}$- and $Q_z$-value that more closely match theoretical expectations, and can also result in a high-quality policy.
\begin{wrapfigure}{r}{0.5\textwidth}
    \centering
    \subcaptionbox{Impact of $F$ Networks}{
    \includegraphics[width=0.25\textwidth]{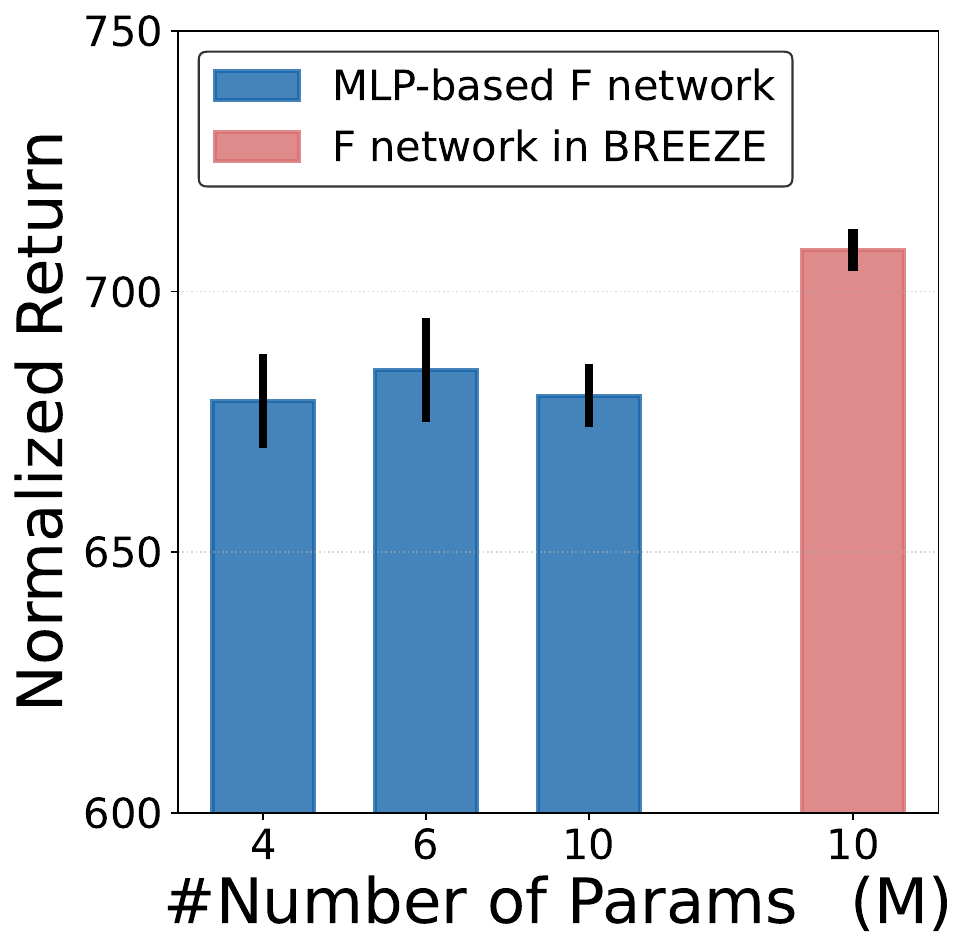}
  }%
  % \hspace{0.01\linewidth}
  \subcaptionbox{Impact of $B$ Networks}{
    \includegraphics[width=0.25\textwidth]{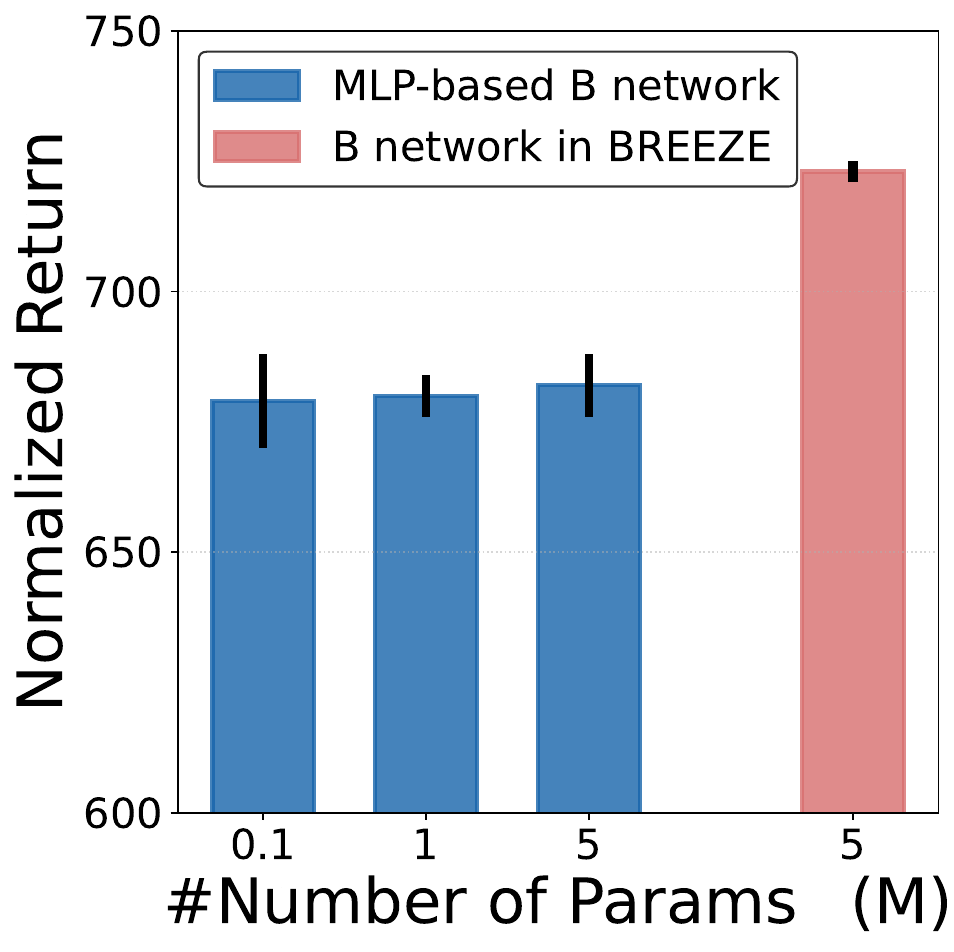}
  }
    \caption{\small Investigation of the modeling architecture on the \textit{vanilla FB}~\cite{Touati2022DoesZR}. We utilize the peak observed in the training stage of each run as a proxy for the architecture's capability ceiling.
    }
    \label{fig:scaling_experiments}
    \vspace{-10pt}
\end{wrapfigure}
The challenges in estimating $M^{\pi_z}$- and $Q_z$-value can be primarily attributed to the limited supervisory signal in the vanilla FB learning objective (Eq.~(\ref{fb_loss})), which fails to constrain the scale of the learned representations adequately. Additionally, modeling successor measures and zero-shot policies across all possible task vectors $z$ introduces substantial complexity, necessitating models with high expressive capacity. While the marginal gains from MCFB’s regularizer suggest its potential utility, a more effective formulation appears necessary. These observations are further corroborated by the results from our method, which integrates a refined OOD regularization strategy and employs a more expressive model architecture. As a result, the distributions of $M^{\pi_z}$ and $Q_z$ values are primarily confined to a reasonable range. Importantly, as illustrated in Figure \ref{fig:scaling_experiments}, these improvements are not attributable to increased model capacity alone: simply scaling up the original MLP-based FB networks does not yield measurable gains. In contrast, the architectural modifications designed in this work lead to a marked improvement in performance, underscoring the influence of model expressivity on the effectiveness of zero-shot learning.
\section{Behavior-Regularized Zero-Shot RL with Expressivity Enhancement}
\label{sec:method}
Motivated by these findings, we propose BREEZE, a framework designed to mitigate the OOD and expressivity issues in FB-based methods. We begin by imposing a behavioral regularizer on the offline zero-shot RL problem, followed by the transformation of the learning process into an in-sample weighted regression paradigm, which can naturally integrate powerful generative models to enhance policy capability. Recognizing the importance of expressivity for learning arbitrary tasks, we describe our practical architecture for stronger representation modeling.

\subsection{Behavior-regularized Optimization}
\label{sec:behavior_constraint}
The main OOD extrapolation issue in the vanilla FB learning losses $\mathcal{L}_{FB}$ and $\mathcal{L}_F$ in Eq.~(\ref{fb_loss}) and Eq.~(\ref{f_loss}) lies in the policy-generated actions $\hat{a}_{t+1}\sim \pi_z(s_{t+1})$ in $F(s_{t+1}, \pi_z(s_{t+1}),z)$ during offline learning. Since $\pi_z$ is trained to maximize $F(s, a, z)^\top z$, it may produce actions that appear optimal under the learned representation but are actually OOD and overestimated, leading to extrapolation errors. Our solution to this is to introduce two designs: 1) stabilize the $F(s_{t+1}, \pi_z(s_{t+1}),z)$ value estimate in the $\mathcal{L}_{FB}$; 2) ensure that $\pi_z$ is better regularized by the dataset samples during policy optimization.

\paragraph{Behavior-regularized representation guidance.} To stabilize the learning of representation learning, we consider a regularized version of $\mathcal{L}_F$ in Eq.~(\ref{f_loss}) by introducing the task-conditioned state-value function $V$ with respect to policy $\pi$, given by:
\begin{equation}
\begin{gathered}
V_{\pi_z}(s,z) := {\max_{\substack{a \in A\\ \mathrm{s.t.} \, \mu(a|s) > 0}}} F(s, a, z)^\top z,
\end{gathered}
\end{equation}
where $z$ is the task vector and $\mu(a|s)$ is the behavior policy in the dataset $\mathcal{D}$. Compared with the traditional formulation of the state-value function, this optimization problem aims to learn an optimal task-conditioned $V$-function solely from dataset samples without explicitly estimating $\mu(a|s)$. It can be solved through a class of value objectives according to different regularization formulations~\citep{Xu2023OfflineRW,garg2023extreme,Kostrikov2021OfflineRL}. In this work, we leverage the commonly used expectile regression as in IQL~\citep{Kostrikov2021OfflineRL}:
\begin{equation} \label{eql:vloss}
\mathcal{L}_{V_{\pi_z}} = \mathbb{E}_{(s,a) \sim \mathcal{D}, z \sim \mathcal{Z}} \left[ L_{2}^{\tau} \left(F(s, a, z)^\top z - V_{\pi_z}(s,z)\right) \right],
\end{equation}
where $ L_{2}^{\tau}(u) = |\tau - \mathbb{I}(u<0)|u^2$, with $\tau > 0.5$ serving as the expectile parameter. By minimizing the above equations, we can obtain a well-regularized and nearly optimal state-value function $V_{\pi_z}$ for different tasks. Given the reward function $r_z$, associated with task $z$, we have the following modified policy evaluation operator $\mathcal{T}^\pi$ given by
\begin{equation}
\begin{gathered}
(\mathcal{T}^\pi Q_z)(s, a, z) := \mathbb{E}_{s'\sim \mathcal{P}(s'|s,a)}\left[r_z(s') + \gamma V_{\pi_z}(s',z)\right],
\end{gathered}
\end{equation}
This formulation of $\mathcal{T}^\pi$ implicitly introduces regularization to the $Q$-function. Following above, we modified $\mathcal{L}_{F}$ in Eq.~(\ref{f_loss}) as following $\mathcal{L}_{F\text{-reg}}$:
\begin{equation}\label{eq:constraint}
\scalebox{0.9}{$\displaystyle
    \mathcal{L}_{F\text{-reg}} = \mathbb{E}_{\substack{(s, a, s') \sim \mathcal{D}}} \left[ \left( F(s, a, z)^\top z - B(s')^\top \mathbb{E}_D[BB^\top]^{-1}z -\gamma V_{\pi_z}(s',z) \right)^2 \right].
$}
\end{equation}

By substituting the potentially unstable target $Q$-approximation $F(s_{t+1}, \pi_z(s_{t+1}),z)^\top z$ with a more well-behaved state-value function $V_{\pi_z}$, we can greatly stabilize the learning of representation, as well as respond to optimality demands simultaneously. In comparison with the direct constraint on unseen value approximation, the behavioral regularization guidance ensures maximum preservation of the representation structure, while being flexible for tuning the degree of conservatism.

\paragraph{Behavior-regularized policy extraction.}
Having introduced regularization for the value function, we now turn to policy learning, which prior work suggests can be equally or even more critical for final performance~\citep{park2024value}. A good policy should effectively leverage high-quality behavior and further generalize near data distribution. To ensure that the policy is well regularized by the dataset samples, we replace the policy objective in Eq.~(\ref{eq:fb}) with the following behavior-regularized optimization problem:
\begin{equation} 
\label{obj}
\max_{\pi_z} \mathbb{E}_{a \sim \pi_z(\cdot|s)} \left[ F(s, a, z)^\top z - V_{\pi_z}(s,z) \right] \quad \text{s.t.} \ \int_a \pi_z(s) da = 1, \, \forall s \in \mathcal{S}, \quad D_{\text{KL}}(\pi_z \| \mu) \leq \epsilon,
\end{equation}
where the term $F(s,a,z)^\top z - V_{\pi_z}(s,z)$ acts as an advantage function maximization, which is equivalent to maximizing the $Q$-function. The KL constraint anchors $\pi_z$ to the behavior policy $\mu$, thereby preventing distributional shift in the offline setting. Given the optimal representation $F$ and $V_{\pi_z}$, we can obtain the closed-form solution for the constrained optimization objective Eq.~(\ref{obj}) by deriving the Lagrangian objective with respect to the policy and setting its derivative to zero, as shown in Proposition \ref{pro:weight_bc} below:
\begin{proposition}\label{pro:weight_bc}
The solution to the constrained optimization problem in Eq.~(\ref{obj}) yields an optimal policy of the form (See Appendix \ref{appendix:Theoretical_Interpretations} for proof):
\begin{equation} \label{weight_bc}
\pi_z^*(s) \propto \mu(a|s) \exp\left(\alpha \cdot (F(s, a, z)^\top z - V_{\pi_z}(s, z))\right), 
\end{equation}
where $1/\alpha$ is the Lagrangian multiplier for the KL constraint.
\end{proposition}
$\alpha$ acts as a temperature that controls the balance between the regularization strength of the behavioral policy and the optimization of the value functions. This closed-form solution elegantly combines the distribution of the behavior policy with the optimal value functions, encouraging the policy to favor high-advantage actions while ensuring all actions remain within the support of the dataset, thus balancing performance with OOD avoidance. With this stable guarantee, we now aim to enhance the policy generalization capability in the next section.

\subsection{Policy Extraction via Task-Conditioned Diffusion Model}
To compute the solution of the weighted behavior cloning objective in Eq.~(\ref{weight_bc}) in practice, a key challenge arises: \textit{how to efficiently model and sample from potentially highly complex and diverse distributions distilled from a re-weighted behavior policy?} 

Direct weighted behavior cloning with a Gaussian policy often fails to capture the complex, multi-modal policy distributions~\citep{Wang2022DiffusionPA,idql,Chen2022OfflineRL}, which is a demand for arbitrary task learning. This motivates the use of diffusion models~\citep{diffusion2015Sohl,ho2020DDPM,score2023songyang}, known for their capacity to learn complex distributions through iterative denoising. Rather than relying on guidance-based techniques that introduce a separate, time-dependent term to steer the sampling process, we draw inspiration from recent advances in weighted regression for diffusion models~\cite{Chen2022OfflineRL,idql,Kang0DPY23,Zheng2024fisor}. Building on the theoretical foundation of prior work~\citep{Zheng2024fisor}, we formalize in Proposition~\ref{thm:weighted-regression} how the optimal policy $\pi_z^*$ can be extracted using a diffusion model trained with a weighted regression objective:

\begin{proposition}[Task-conditioned diffusion policy extraction via weighted regression]
\label{thm:weighted-regression}
The extraction of optimal policy $\pi_z^*$ in Eq.~(\ref{weight_bc}) can be achieved by (i) minimizing the weighted regression loss defined as:
\begin{equation}
\label{eq:diffusion_objective}
\min_{\theta} \mathbb{E}_{t\sim\mathcal{U}([0,T]), \epsilon\sim\mathcal{N}(0, I),(s,a)\sim \mathcal{D}} \Big[\exp\Big( \alpha \cdot(F(s,a, z)^\top z - V_{\pi_z}(s,z))\Big) \|\epsilon - \epsilon_{\theta,z}(a_t,s,z,t)\|_2^2 \Big],
\end{equation}
with expectations taken over $t \sim \mathcal{U}([0,T])$, $\epsilon \sim \mathcal{N}(0,I)$, and $(s,a) \sim \mathcal{D}$, where $a_t = \alpha_t a + \sigma_t \epsilon$ follows the forward process $\mathcal{N}(a_t|\alpha_t a,\sigma_t^2 I)$ parameterized by noise schedules $\alpha_t, \sigma_t$ and $z \sim \mathbb{R}^d$ denotes the corresponding task;
and (ii) sampling from $\pi_z^*$ by solving the corresponding diffusion ODEs/SDEs with the learned $\epsilon_{\theta,z}$. (See Appendix~\ref {appendix:Theoretical_Interpretations} for the detailed proof.)
\end{proposition}
With the above objective, we can avoid the need for learning the additional time-dependent guidance term and thereby reduce the complexity, while getting a stable policy.

For action selection, we employ a rejection sampling mechanism to boost policy performance. Specifically, we first sample \(K\) candidate actions \(\{a^{(1)}, \ldots, a^{(K)}\} \sim \pi_{z}(s)\) through the policy rollout. We then evaluate each candidate using the $Q$-function (approximated by $F(s,a,z)^\top z$) and select the action with the highest value:
\begin{equation} \label{eq:topk}
a^\star \triangleq \arg\max_{a \in \{a^{(1)}, \ldots, a^{(K)} \sim \pi_z(s)\}} F(s,a,z)^\top z.
\end{equation}
This two-stage approach ensures the policy balances both conservatism and optimism. The diffusion model generates diverse, in-distribution candidates, while the selection step identifies the action with the highest expected return. This combination of expressive generative modeling and value-based selection is crucial for achieving robust, high-performance zero-shot generalization.

\subsection{Expressivity Enhancement for Representation Modeling}

The effectiveness of any policy—particularly a highly expressive diffusion policy—in leveraging learned action weights hinges on accurate value estimation. Biased value estimates can degrade policy learning to simple imitation, failing to redistribute probability mass toward superior actions. To fully exploit the representational capacity of our diffusion policy, the underlying value representations must precisely capture complex task and dynamic relationships. We thus introduce enhanced network architectures for the forward ($F$) and backward ($B$) representations.

\begin{wrapfigure}{r}{0.52\textwidth}
    \centering
    \vspace{-10pt}
    \includegraphics[width=0.55\textwidth]{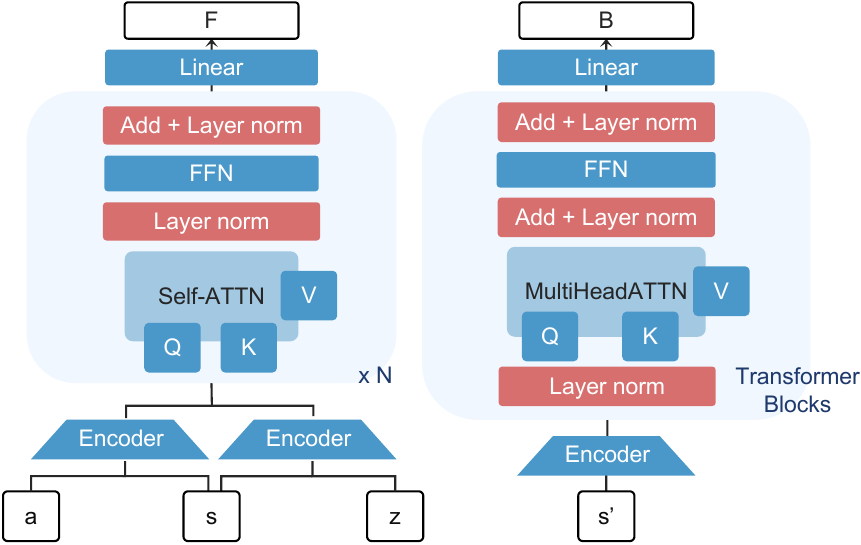} 
    \caption{ \small Architecture of BREEZE}
    \label{fig:architecture}
    \vspace{-20pt}
\end{wrapfigure}

As shown in Section~\ref{sec:motivation}, we empirically identify a pair of attention-based architectures that improve zero-shot performance over the original FB implementations. Below, we detail the design of our networks, illustrated in Figure~\ref{fig:architecture}.

\paragraph{Forward Network.}
The forward network encodes state-task and state-action pairs using two separate linear encoders, following Touati et al.~\citet{Touati2022DoesZR}. The state-task encoder captures task-conditioned dynamics, while the state-action encoder extracts agent behavior patterns. These two feature sets correspond to semantically distinct concepts in the MDP. To accurately approximate the measure $M(s, a, s', z)$ in Eq.~\ref{eq:fb}, the representation $F$ must integrate information from both behavioral and task contexts, capturing their interdependencies effectively. We therefore model the two encoded features as a length-2 embedding sequence and process them through self-attention blocks. This allows bidirectional feature refinement between task conditions and agent behaviors. The resulting representation is projected onto a $d$-dimensional space via linear layers.

\paragraph{Backward Network.} As defined in Eq.~\ref{eq:fb}, the backward representation $B$ encodes state-level structural information, serving as a global embedding of the environment. Intuitively, $B$ should employ a more complex architecture, maintaining orthogonality and enforcing $F$ alignment. We equip $B$ with a stack of standard transformer-based networks with multi-head attention. The final output is projected onto a $d$-dimensional space, consistent with the forward network.

These architecturally enhanced representations enable BREEZE to model complex relationships more accurately, leading to improved value estimates and policy performance. Further implementation details are provided in Appendix~\ref{appendix:Architectures}.

\begin{table}[t]
\centering
\begin{minipage}{\textwidth}
\centering
\caption{\small \textbf{IQM results on ExORL benchmark.} We report the best overall performance when all tasks perform well. Each value is averaged over 10 rollouts on 5 random seeds.}
\resizebox{0.92\linewidth}{!}{\scriptsize
\begin{tabular}{lc|cccccc}
\toprule
\textbf{Dataset} & \textbf{Domain} & \textbf{SF-LAP} & \textbf{FB} & \textbf{VCFB} & \textbf{MCFB} & \textbf{HILP} & \textbf{BREEZE}\\ 
\midrule
\multirow{3}{*}{RND}    & Walker    & $516 \pm 97$  & $661 \pm 10$   & $653 \pm 22$  & $659 \pm 51$  & $665 \pm 33$  & \colorbox{mine}{$693 \pm 16$} \\ [0.4em]
                        & Jaco      & $18 \pm 18$  & $32 \pm 23$   & $46\pm 35$    & $41 \pm 34$   & $52 \pm 21$    & \colorbox{mine}{$84 \pm 14$} \\ [0.4em]
                        & Quadruped & $330 \pm 165$ & $671 \pm 14$  & $609 \pm 29$   & $684 \pm 18$   & $674 \pm 28$  & \colorbox{mine}{$725 \pm 23$} \\ 
\cmidrule(lr){1-8}
\multirow{3}{*}{APS}    & Walker    & $324 \pm 24$ & $517 \pm 99$  & $487 \pm 75$  & $578 \pm 35$  & \colorbox{mine}{$643 \pm 22$}   & $637 \pm 21$ \\ [0.4em]
                        & Jaco      & $39 \pm 26$ & $22 \pm 14$    & $20 \pm 18$    & $22 \pm 3$    & $84 \pm 16$   & \colorbox{mine}{$132 \pm 16$} \\ [0.4em]
                        & Quadruped & $498 \pm 160$ & $668 \pm 29$  & $664 \pm 3$   & $659 \pm 50$   & $679 \pm 14$   & \colorbox{mine}{$698 \pm 24$} \\  
\cmidrule(lr){1-8}
\multirow{3}{*}{PROTO}  & Walker    & $382 \pm 129$ & $650 \pm 19$   & $611 \pm 94$  & $602 \pm 112$  & \colorbox{mine}{$715 \pm 31$}  & $663 \pm 19$ \\ [0.4em]
                        & Jaco      & $15 \pm 14$ & $21 \pm 26$    & $13 \pm 12$    & $20 \pm 21$   & $44 \pm 19$   & \colorbox{mine}{$74 \pm 26$} \\ [0.4em]
                        & Quadruped & $199 \pm 10$ & $222 \pm 107$  & $185 \pm 72$  & $219 \pm 135$  & $216 \pm 54$   & \colorbox{mine}{$389 \pm 44$} \\
\cmidrule(lr){1-8}
\multirow{3}{*}{DIAYN}  & Walker    & $239 \pm 79$ & $338 \pm 74$   & $268 \pm 67$  & $268 \pm 97$  & $461 \pm 64$  & \colorbox{mine}{$463 \pm 42$} \\ [0.4em]
                        & Jaco      & $32 \pm 26$  & $22 \pm 6$    & $24 \pm 3$    & $15 \pm 1$    & $52 \pm 7$    & \colorbox{mine}{$78 \pm 11$} \\ [0.4em]
                        & Quadruped & $207 \pm 168$ & $562 \pm 23$  & $511 \pm 37$  & $643 \pm 14$  & \colorbox{mine}{$670 \pm 4$}   & $666 \pm 2$ \\
\bottomrule
\end{tabular}}
\label{tab:iqmresult}
\end{minipage}

\vspace{0.5em}

% \begin{figure}[htbp]
\begin{minipage}[t]{0.73\textwidth}
    \centering
    \captionof{table}{\small \textbf{IQM results with 100k-dataset.} Experiments on randomly sampled 100k-transition data from each dataset on the ExORL benchmark}
    \resizebox{0.99\linewidth}{!}{\scriptsize
    \begin{tabular}{lc|cccc}
    \toprule
    \textbf{Dataset} & \textbf{Domain} & \textbf{FB} & \textbf{VCFB} & \textbf{MCFB}  & \textbf{BREEZE (ours)}\\ 
    \midrule
    \multirow{3}{*}{RND}    & Walker    & $264 \pm 33$  & $350 \pm 29$  & $287 \pm 48$  & \colorbox{mine}{$525 \pm 13$} \\ [0.4em]
                            & Jaco      & $7 \pm 5$     & $9 \pm 2$     & $13 \pm 7$    & \colorbox{mine}{$36 \pm 5$} \\ [0.4em]
                            & Quadruped & $176 \pm 123$ & $233 \pm 52$  & $123 \pm 61$  & \colorbox{mine}{$474 \pm 21$} \\ 
    \cmidrule(lr){1-6}
    \multirow{3}{*}{APS}    & Walker    & $370 \pm 66$  & $416 \pm 10$  & $389 \pm 77$  & \colorbox{mine}{$539 \pm 15$} \\ [0.4em]
                            & Jaco      & $21 \pm 17$   & $14 \pm 13$   & $29 \pm 27$   & \colorbox{mine}{$38 \pm 9$} \\ [0.4em]
                            & Quadruped & $340 \pm 29$  & $351 \pm 57$  & $318 \pm 122$ & \colorbox{mine}{$556 \pm 52$} \\  
    \cmidrule(lr){1-6}
    \multirow{3}{*}{PROTO}  & Walker    & $415 \pm 19$  & $513 \pm 31$  & $463 \pm 11$  & \colorbox{mine}{$553 \pm 18$} \\ [0.4em]
                            & Jaco      & $16 \pm 2$    & $18 \pm 12$   & $12 \pm 7$    & \colorbox{mine}{$29 \pm 12$} \\ [0.4em]
                            & Quadruped & $198 \pm 111$ & $106 \pm 103$ & \colorbox{mine}{$240 \pm 134$}  & $181 \pm 60$ \\
    \cmidrule(lr){1-6}
    \multirow{3}{*}{DIAYN}  & Walker    & $202 \pm 94$  & $210 \pm 81$  & $196 \pm 26$  & \colorbox{mine}{$330 \pm 43$} \\ [0.4em]
                            & Jaco      & $17 \pm 11$   & $18 \pm 5$    & $20 \pm 26$   & \colorbox{mine}{$22 \pm 15$} \\ [0.4em]
                            & Quadruped & $295 \pm 46$  & $288 \pm 48$  & $286 \pm 34$  & \colorbox{mine}{$446 \pm 78$} \\
    \bottomrule
    \end{tabular}}
    \label{tab:smallsamlpeiqmresult}
\end{minipage}
\hfill
\begin{minipage}[t]{0.26\textwidth}
    \vspace{10pt}
    \centering
    \begin{subfigure}[t]{\textwidth}
        \includegraphics[width=0.98\textwidth]{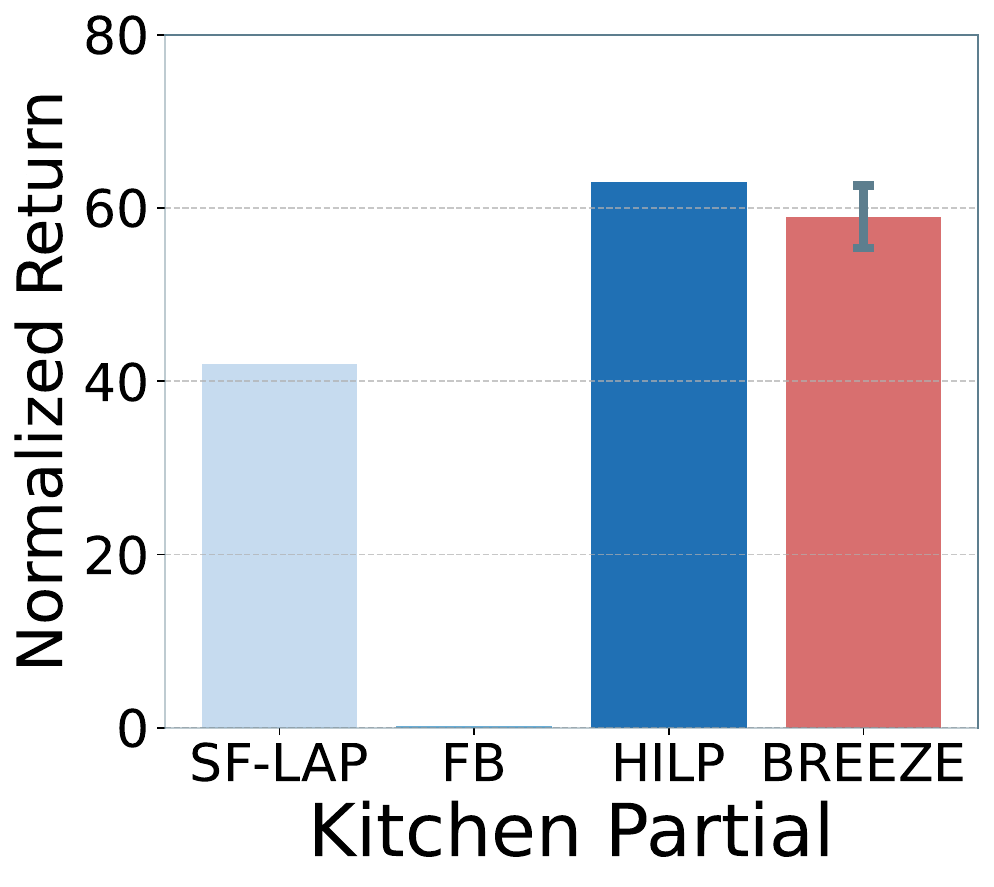}
    \end{subfigure}
    \begin{subfigure}[t]{\textwidth}
        \includegraphics[width=0.98\textwidth]{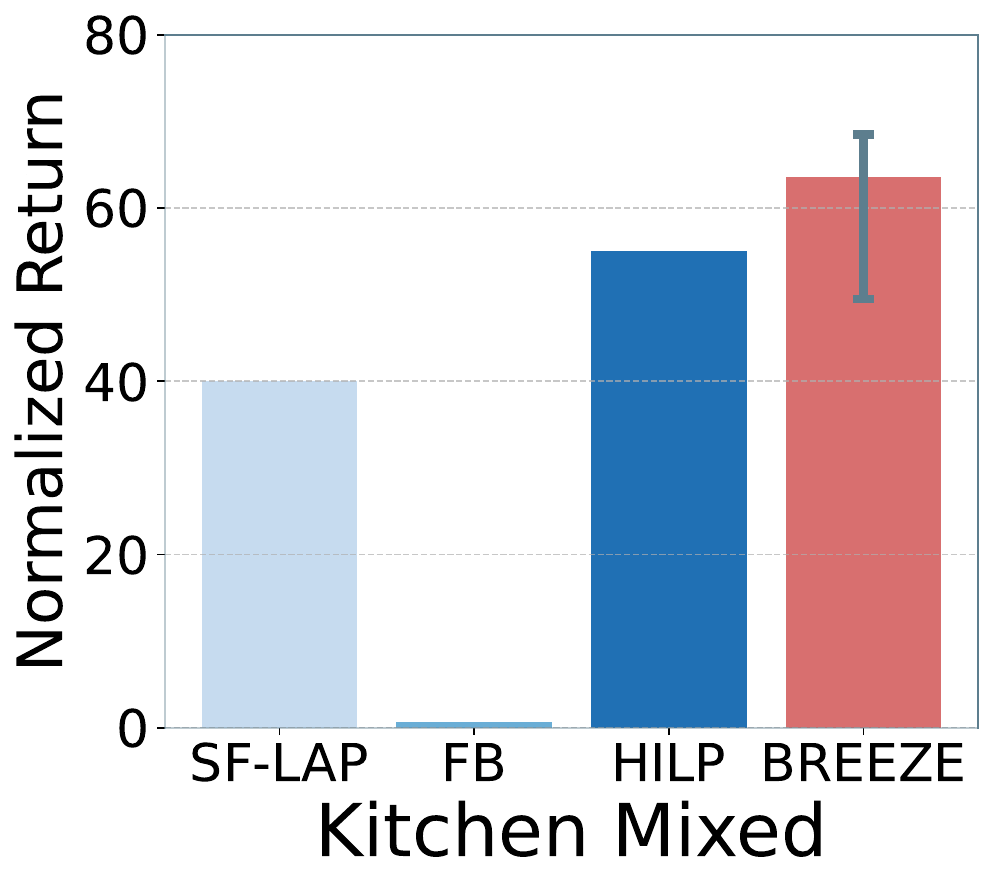}
    \end{subfigure}
    \captionof{figure}{\label{fig:kitchen_result}\small Normalized scores on Kitchen environment}
\end{minipage}
\end{table}

\section{Experiments}
We conduct extensive evaluations of BREEZE against previous offline zero-shot RL algorithms across various challenging domains with distinct tasks. We also present the ablations on component choices and hyperparameters.

\paragraph{Environmental setup.} Our main experiments are conducted on the ExORL benchmark~\citep{yarats2022dontchangealgorithmchange}, which provides a variety of datasets collected by several unsupervised RL algorithms~\citep{laskin2021URLB}. We select datasets collected by 4 algorithms: \textit{RND}~\citep{Burda2018ExplorationBR}, \textit{APS}~\citep{Liu2021APSAP}, \textit{DIAYN}~\citep{Eysenbach2018DiversityIA}, and \textit{PROTO}~\citep{Yarats2021ReinforcementLW}. The experiments span 3 domains and 12 tasks (Walker: Stand, Walk, Run, Flip; Jaco: Reach Top/Bottom Left/Right; Quadruped: Stand, Walk, Run, Jump), bringing the total to 48 state-based complex tasks for performance calculation. In addition, we consider four challenging multi-stage tasks in the D4RL~\citep{fu2021d4rldatasetsdeepdatadriven} Franka Kitchen domain~\citep{lynch2019play} with two datasets (mixed and partial), which require long-horizon sequential manipulation on 4 subtasks. The overall setup involves 2 locomotion domains and 2 manipulation (goal-reaching) domains. For all goal-conditioned domains, we use the backward representation to calculate the goal features as the task vector $z$.

\paragraph{Baselines.} We only consider offline algorithms that could perform a zero-shot policy generalization for our evaluation. Our baselines include: 1) \textit{SF-LAP}~\citep{Barreto2016SuccessorFF,Borsa2018UniversalSF}: successor feature-based method with basic features to be \textit{Laplacian Eigenfunction (LAP)}~\citep{wu2018laplacian}, 2) \textit{vanilla FB}~\citep{Touati2021FB}, 3) \textit{VCFB} and \textit{MCFB}~\citep{Jeen2023ZeroShotRL}: FB-based offline algorithms with \textit{CQL}-style regularizer~\citep{kummar2020cql} on successor measure $M^\pi$ and $Q$-value separately for OOD issue avoidance, 4) \textit{HILP}~\citep{park2024foundationpolicieshilbertrepresentations}: a state-based representation modeling algorithm keeping distance-preserving temporal structures in latent space having zero-shot capability.

\paragraph{Evaluation.} Our experiments are designed to evaluate two key aspects: \textbf{zero-shot performance} and \textbf{robustness}. To assess the overall zero-shot capability, we experiment on the full transitions of each dataset on ExORL and D4RL Kitchen. For robustness, we first present performance curves across training time to demonstrate learning stability and then conduct experiments on a uniformly subsampled 100,000 transitions to examine performance in a limited and low-quality data coverage scenario, following the evaluation process of~\citet{Jeen2023ZeroShotRL}. We record the Interquartile Mean (IQM), a robust metric that reduces the influence of outliers. We evaluate the checkpoints every 10,000 updates with 10 rollouts for each experiment. On Kitchen, we multiply the returns by 25 for normalization, following~\citet{park2024foundationpolicieshilbertrepresentations}.
Further implementation and experimental details are provided in Appendix~\ref{appendix:Experimental_details}.

\subsection{Experimental Results}
We analyze the findings of our experiments through the following key questions:
\begin{itemize}[leftmargin=*]
\item \textit{Does BREEZE possess zero-shot capability for novel tasks?}\quad
    Following previous studies~\citep{agarwal2021deep,Jeen2023ZeroShotRL}, we record the highest aggregated mean performance across random seeds in the offline pretrain stage. Table \ref{tab:iqmresult} reports this score across all tasks on each domain of ExORL, averaged on 5 random seeds, with standard deviations. BREEZE achieves the best or near-best returns in most domains, demonstrating superior zero-shot generalization capability. Moreover, BREEZE can enhance \textit{vanilla FB}'s generalization capability in long-horizon tasks. As shown in Figure~\ref{fig:kitchen_result}, we draw the box with the performance reported in the previous study~\cite{park2024foundationpolicieshilbertrepresentations}, and record BREEZE's top averaged returns across 4 random seeds. While vanilla FB struggles in these long-horizon tasks, BREEZE achieves significantly higher performance, suggesting its ability to better leverage the dataset for complex sequential manipulation.

    More results are provided in Appendix \ref{appendix:additional_result}.

\item \textit{How does BREEZE quantitatively compare to baselines in terms of stability and convergence speed?}\quad
    We provide the learning curves in \textit{RND}~\citep{Burda2018ExplorationBR} datasets as Figure \ref{fig:curves}. In locomotion domains (i.e., Quadruped and Walker), BREEZE converges faster to higher performance, with smoother curves and lower variance (as shown by the small shaded area). In the manipulation domain (i.e., Jaco), BREEZE substantially outperforms all baselines with a higher learning speed. Overall, BREEZE demonstrates faster convergence and enhanced stability across both locomotion and manipulation scenarios.
    
    % Refer to Appendix~\ref{appendix:curves} for more training time performance curves.
    
\item \textit{How does BREEZE perform facing different quality, diversity of datasets?}\quad
    Table~\ref{tab:smallsamlpeiqmresult} shows results on 100,000-transition subsets of ExORL datasets. BREEZE maintains a pronounced advantage over FB-based baselines in this small-sample regime. Compared to explicit constraint methods (MCFB/VCFB), these results highlight the importance of behavior alignment across different data regimes. BREEZE also exhibits superior stability, as shown in Figure~\ref{fig:curves_small}.
    
\end{itemize}
Refer to Appendix~\ref{appendix:curves} for more learning curves.

\begin{figure*}[t]
\centering
\includegraphics[width=1\linewidth]{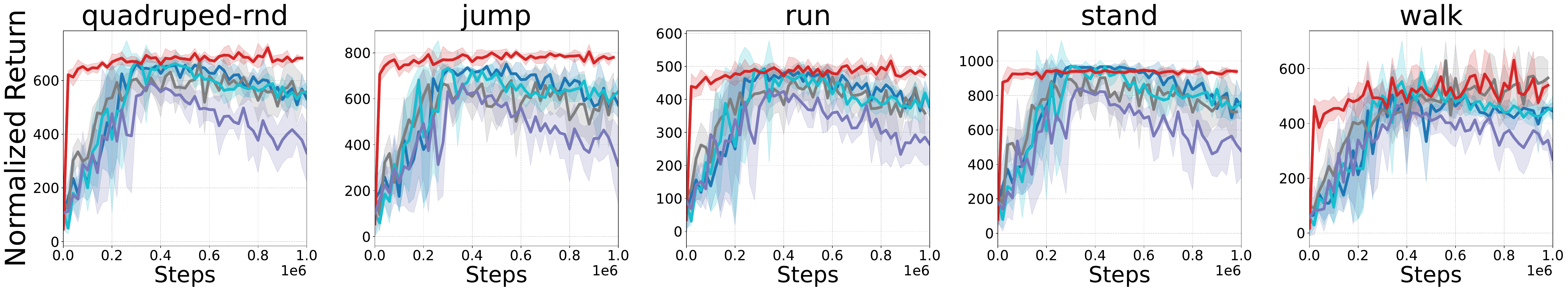}
\includegraphics[width=1\linewidth]{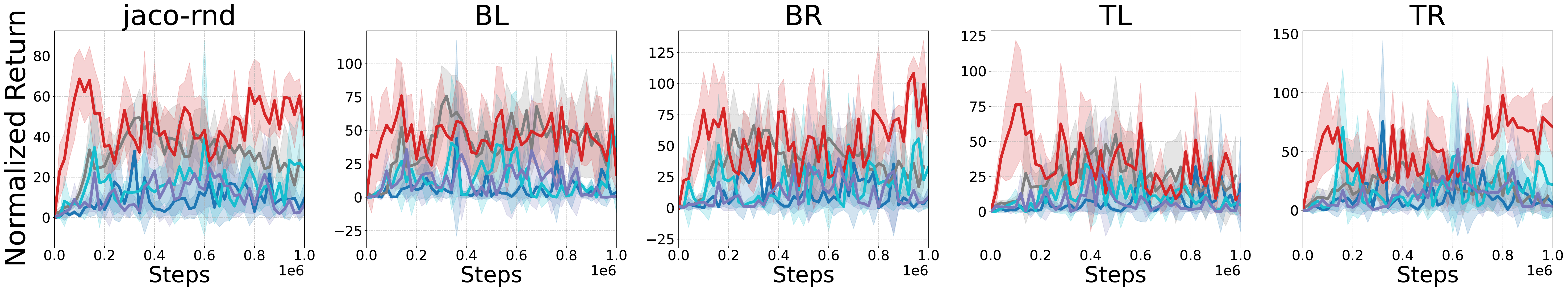}
\includegraphics[width=1\linewidth]{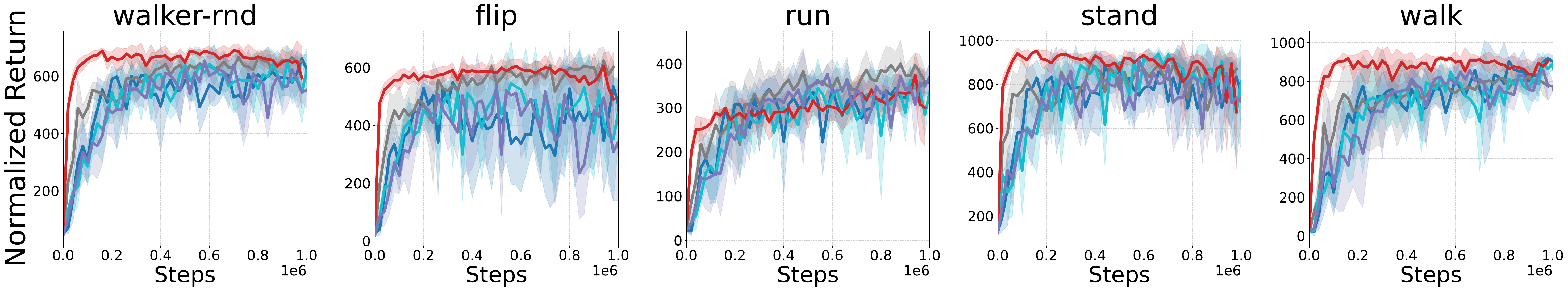}
\includegraphics[width=0.9\linewidth]{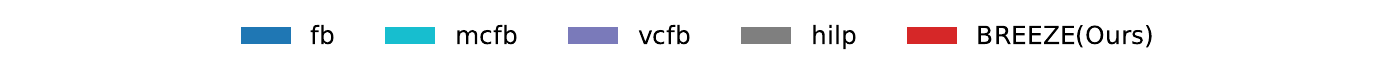}
\captionsetup{skip=3pt}
\caption{\label{fig:curves} \small \textbf{Learning Curves on ExORL RND.} The solid lines represent the average return over 5 random seeds, while the shaded area denotes the standard deviation.} 
\end{figure*}
\begin{figure*}[t]
\centering
\includegraphics[width=1\linewidth]{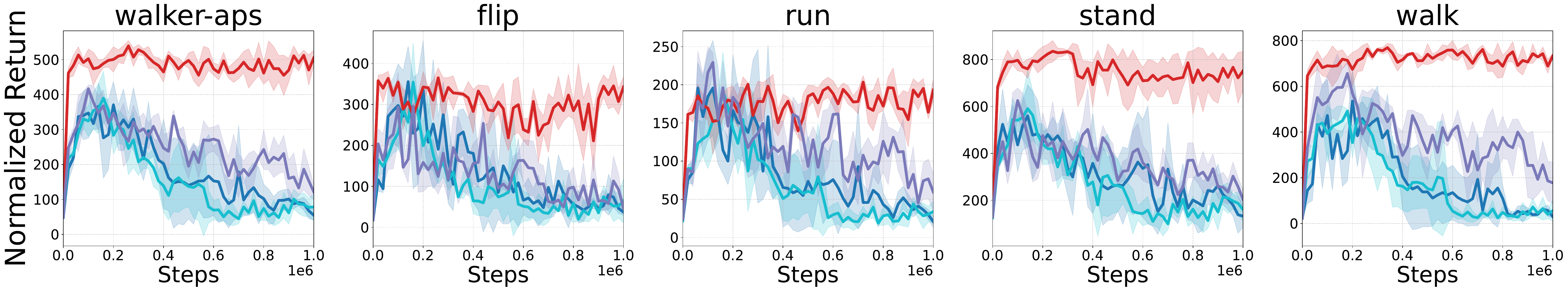}
\includegraphics[width=0.9\linewidth]{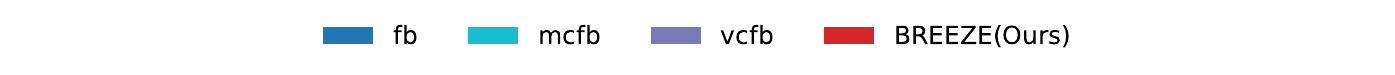}
\captionsetup{skip=3pt}
\caption{\label{fig:curves_small} \small \textbf{Learning Curves on 100k-subsample transitions on ExORL APS Walker domain.} The solid lines represent the average return over 3 random seeds, while the shaded area denotes the standard deviation.} 
\end{figure*}

\begin{minipage}[t]{1\textwidth}
\begin{minipage}{0.55\textwidth}
  \centering
  \resizebox{1\linewidth}{!}{\scriptsize
  \begin{tabular}{l|ccc} 
  \toprule
  \textbf{Dataset}  & \textbf{w/o FB Enhancement} & \textbf{w/o Diffusion} & \textbf{BREEZE} \\ 
  \midrule
  Walker-RND        & 646$\pm$18    & \textbf{707$\pm$13}   & 693 \\ 
  Jaco-RND          & 80$\pm$24     & 62$\pm$7              & \textbf{84} \\ 
  Quadruped-RND     & 685$\pm$13    & 530$\pm$33            & \textbf{725} \\
  \cmidrule(lr){1-4}
  Walker-APS        & 614$\pm$48    & 587$\pm$53        & \textbf{637} \\ 
  Jaco-APS          & 82$\pm$13     & 45$\pm$50         & \textbf{132} \\ 
  Quadruped-APS     & 655$\pm$20    & 568$\pm$19        & \textbf{698} \\ 
  \bottomrule
  \end{tabular}}
  \captionof{table}{\label{table:component_abl}Ablation on necessity of FB enhancement and diffusion policy (3 seeds).}
 % \vspace{10pt}
\end{minipage}%
\hspace{8pt}
\begin{minipage}{0.2\textwidth}
\vspace{-5pt}
\centering
\includegraphics[width=0.99\textwidth]{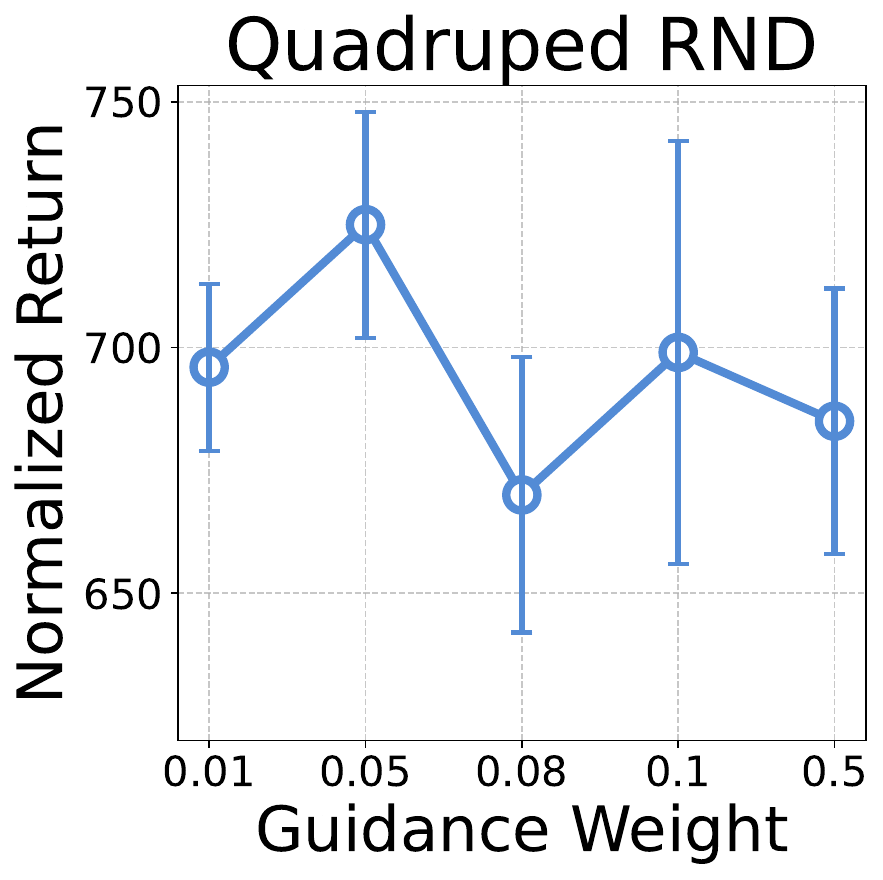}
\captionof{figure}{\label{img:alpha_abl}Ablation on the temperature.}
\end{minipage}%
\hspace{5pt}
\begin{minipage}{0.2\textwidth}
\vspace{-5pt}
\centering
\includegraphics[width=0.99\textwidth]{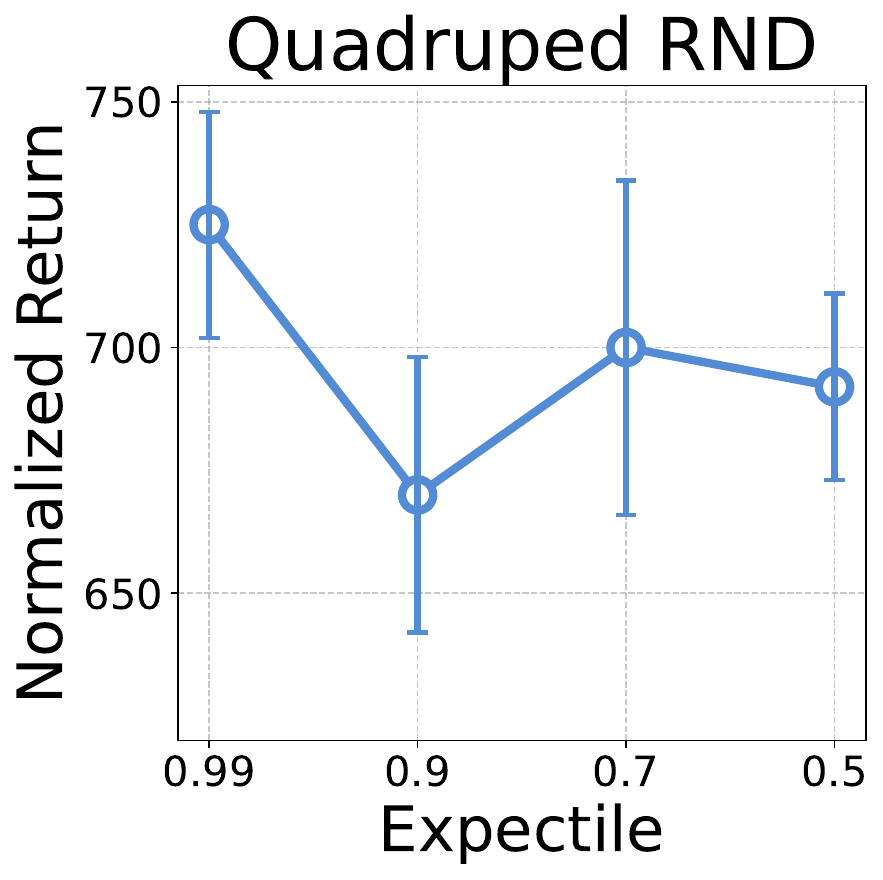}
\captionof{figure}{\label{img:tau_abl}Ablation on the expectile factor.}
\end{minipage}%
% \vspace{-15pt}
\end{minipage}%

\subsection{Ablation Studies}
\paragraph{Hyperparameter choices.}
We ablate two key hyperparameters in our behavior-constrained algorithm: the expectile value $\tau$ and the temperature $\alpha$. As shown in Figure \ref{img:alpha_abl} and Figure~\ref{img:tau_abl} (Quadruped domain, RND dataset, 3 random seeds), performance improves near monotonically with $\tau$ and peaks at $\alpha=0.05$, and similarly to other domains. This indicates that, while in-sample conservatism is crucial, the algorithm benefits from more aggressive, value-weighted optimization. Consequently, we select $\tau=0.99$ and $\alpha=0.05$ as defaults in ExORL, optimally balancing conservatism with performance for downstream tasks.

\paragraph{Necessity of each component.} We ablate our two core innovations—an enhanced representation model and a diffusion policy—to assess their necessity in RND and APS datasets. BREEZE is ablated using: a gaussian policy variant (\textit{w/o Diffusion}) and the baseline architecture variant (\textit{w/o FB Enhancement}). We process transitions with data normalization for the diffusion policy, as in a commonly used preprocessing step in previous diffusion-based methods. We cancel this process and return to the default policy setting to have a reasonable comparison of the components' contributions. As shown in Table~\ref{table:component_abl}, each component somehow individually slightly boosts performance, but their synergy creates a leap in zero-shot capability. This demonstrates a critical interdependence: the design of BREEZE is a logical combination to handle the system issue.
\section{Related Works}
\paragraph{Unsupervised zero-shot RL.}
Recent advancements in unsupervised zero-shot RL \cite{laskin2021URLB, park2024foundationpolicieshilbertrepresentations, levine2024FRE} mainly build upon the utilization of successor representation (SR)~\citep{SR1993}, with two main branches: successor feature (SF)~\citep{Barreto2016SuccessorFF,Borsa2018UniversalSF} and forward-backward representation (FB)~\citep{Touati2021FB,Touati2022DoesZR}. Both approaches require decomposing and linearizing reward-aware dynamics to enable zero-shot inference for arbitrary downstream rewards. Other approaches relax the linearity assumption by modeling SR as full distributions using generative models~\citep{gammamodel,GHM,distributionSR,Farebrother2025TemporalDF}. Among all, FB has emerged as a pivotal framework, factorizing successor measures into forward and backward components and avoiding the potential collapse issue in SF at the same time. Extensions like MCFB/VCFB~\citep{Jeen2023ZeroShotRL} build on offline RL, enforcing explicit value constraints to mitigate the over-estimation issues that arise from updating on OOD rollouts. FB-CPR~\citep{Tirinzoni2025ZeroShotWH} introduces a regularization of the $Q$-estimation by imitating pre-collected datasets in an online setting. FB-AWARE~\citep{Cetin2024FinerBF} introduces an autoregressive mechanism to handle the over-linearization problem in vanilla FB. Unlike SR-based approaches, HILP~\citep{park2024foundationpolicieshilbertrepresentations} structures the latent space so that distances correspond to transition times between states, enabling hierarchical, goal-conditioned, and zero-shot policy learning. Meanwhile, functional-encoder methods~\citep{levine2024FRE,amy2024FE} directly encode task information by pretraining on manually designed reward functions, but are constrained by the need for extensive reward engineering and fine-tuning.

\paragraph{Offline RL.} Numerous studies address offline RL, in which learning is restricted to a fixed dataset without online interaction, leading to a distributional shift. The earliest approaches use policy constraints, enforcing the learned policy to stay close to the behavior policy~\citep{Fujimoto2018OffPolicyDR,kumar2019stabilizing,fujimoto2021minimalist,li2023proto} or the distribution coverage to stay near the dataset~\citep{li2023when,cheng2023look}. Another established line of work employs value regularization, which directly penalizes the value estimates of OOD actions~\citep{kummar2020cql,Kostrikov2021fishercritic,niu2022when,yang2022rorl,xu2022constraints}. Instead of an explicit constraint, in-sample learning methods~\citep{Kostrikov2021OfflineRL,Xu2023OfflineRW,xu2022a,wang2023offline} improve stability by learning values and policies only from state-action pairs within the dataset.

Moving beyond conventional Gaussian policies, recent studies have highlighted the importance of modeling multimodal action distributions~\cite{Wang2022DiffusionPA,idql,Chen2022OfflineRL}. Diffusion models~\citep{diffusion2015Sohl,ho2020DDPM,score2023songyang}, known for their strong multi-modal modeling capability, have been well studied in imitation and trajectory modeling~\citep{pmlr-v162-janner22a,zheng2025diffusionbased,chi2023diffusionpolicy,ni2023MetaDiffuser,tan2025flow}, and have recently been integrated into offline RL frameworks. Some approaches use the value function as energy guidance for sampling~\citep{CEP,Mao2024DiffusionDICEID}. Other works, such as SfBC~\citep{Chen2022OfflineRL}, IDQL~\citep{idql}, and EDP~\citep{Kang0DPY23}, extract a diffusion policy via weighted regression under the in-sample learning framework.
\section{Conclusion}

In this work, we present BREEZE, a novel framework that mitigates the improper scaling issue in existing FB-based zero-shot RL methods. Specifically, BREEZE tackles two critical limitations: offline extrapolation errors and constrained expressivity. Our solution integrates behavior-regularized value estimation with a task-conditioned diffusion policy, enabling stable in-sample learning while capturing complex, multimodal action distributions. Coupled with expressive attention-based representations, BREEZE more accurately models dynamics and value functions. Extensive experiments demonstrate that BREEZE consistently outperforms existing methods across various benchmarks, underscoring the importance of calibrated regularization and sufficient model capacity for zero-shot generalization. While BREEZE demonstrates improved robustness, its primary limitation is the increased computational cost from diffusion-based sampling, which is a common trade-off for utilizing high-performance generative policies. Further discussion is provided in Appendix~\ref{appendix:limitation}.
\section*{Acknowledgment}
This work is supported by the Wuxi Research Institute of Applied Technologies, Tsinghua University, under Grant 20242001120, and funded by Horizon Robotics, AsiaInfo, and the Xiongan AI Institute.

\bibliography{references}
\bibliographystyle{plainnat}

% %%%%%%%%%%%%%%%%%%%%%%%%%%%%%%%%%%%%%%%%%%%%%%%%%%%%%%%%%%%%
\newpage
\appendix
\section{Theoretical Proofs}
\label{appendix:Theoretical_Interpretations}

\subsection{Proof of Proposition \ref{pro:weight_bc}}
\setcounter{proposition}{0}
\begin{proposition}\label{pro:weight_bc-appendix}
The closed-form optimal policy for Eq.~(\ref{obj}) follows the form:
\begin{equation} \label{weight_bc-appendix}
\pi_z^*(s) \propto \mu(a|s) \exp\left(\alpha \cdot (F(s, a, z)^\top z - V_{\pi_z}(s, z))\right), 
\end{equation}
where $\mu$ denotes the behavior distribution of the dataset.
\end{proposition}

\begin{proof}
Construct the Lagrangian with dual variables \(\lambda_1\) and \(\lambda_2\) for the equality and KL-divergence constraints, respectively:
\begin{equation}
\mathcal{L}(\pi_z, \lambda_1, \lambda_2) = \mathbb{E}_{a \sim \pi_z} \left[ F(s, a, z)^\top z - V_{\pi_z}(s, z) \right] - \lambda_1 \left( \int \pi_z(a|s) \, da - 1 \right) - \lambda_2 \left( D_{\text{KL}}(\pi_z \| \mu) - \epsilon \right).
\end{equation}

Taking the functional derivative of \(\mathcal{L}\) with respect to \(\pi_z(a|s)\) and setting it to zero yields:
\begin{equation}
\frac{\partial \mathcal{L}}{\partial \pi_z} = F(s, a, z)^\top z - V_{\pi_z}(s, z) - \lambda_2 \left( \log \pi_z(a|s) - \log \mu(a|s) + 1 \right) - \lambda_1 = 0.
\end{equation}

Solving above for $\pi_z$ and absorbing constant terms into the normalization factor \(Z(s, z)\) gives:
\begin{equation}
\pi_z^*(a|s) = \frac{\mu(a|s)}{Z(s, z)} \exp\left( \alpha \cdot (F(s, a, z)^\top z - V_{\pi_z}(s, z)) \right).
\end{equation}
where $\alpha=1/\lambda_2$ and the partition function is defined as:
\begin{equation}
Z(s, z) = \int \mu(a|s) \exp\left( \alpha \cdot (F(s, a, z)^\top z - V_{\pi_z}(s, z)) \right) \, da,
\end{equation}
This ensures the policy is properly normalized. The proof is completed.
\end{proof}

\subsection{Proof of Proposition \ref{thm:weighted-regression}}
We begin by conducting a theoretical review of diffusion models, and then demonstrate how to derive the optimal policy through task-conditioned weighted regression.

\textbf{Diffusion (probabilistic) models.} Diffusion models~\citep{diffusion2015Sohl,ho2020DDPM,score2023songyang} are powerful generative models designed to learn and replicate complex distributions. Given an $N$-sample dataset $\mathcal{D}=\{x_0^{i}\}^N_{i=1}$ from an unknown distribution $q(x_0)$, diffusion models can accurately approximate $q(x_0)$ and generate new samples from the approximation by the following two processes: 
\begin{itemize}[leftmargin=*]
    \item \textbf{Forward process.} The forward process gradually adds Gaussian noise on sample \( x_0 \) from time 0 to \( x_T \) over \( T \) timesteps, with transformation distribution follows:
    \begin{equation}
   q_{t0}(x_t|x_0) = \mathcal{N}(x_t|\alpha_t x_0, \sigma_t^2 I),
    \end{equation}
    where \( \alpha_t \) and \( \sigma_t \) are predefined noise schedules ensuring \( q_T(x_T|x_0) \approx q_T(x_T) \approx \mathcal{N}(x_T|0, \tilde{\sigma}^2 I) \) for some $\tilde{\sigma}>0$ that is independent of $x_0$.

    \item \textbf{Backward process.} Starting from \( x_T \sim \mathcal{N}(x_T| 0, \tilde{\sigma}^2 I) \), diffusion models reconstruct the original data \( x_0 \) by solving diffusion ODE/SDE \cite{song2021scorebased} from $T$ to 0:
   \begin{equation} \label{eq:ode}
   \text{(Diffusion ODE) }\quad dx_t = \left[ f(t)x_t - \frac{1}{2}g^2(t)\nabla_{x_t}\log q_t(x_t) \right] dt,
   \end{equation}
   \begin{equation} \label{eq:sde}
   \text{(Diffusion SDE) }\quad dx_t = \left[ f(t)x_t - g^2(t)\nabla_{x_t}\log q_t(x_t) \right] dt + g(t)d\tilde{\omega},
   \end{equation}
   where \( f(t) = \frac{d \log \alpha_t}{dt} \), \( g^2(t) = \frac{d\sigma_t^2}{dt} - 2 \frac{d \log \alpha_t}{dt} \sigma_t^2 \)~\citep{kingma2023variationaldiffusionmodels}, and \( \tilde{\omega} \) is a standard Wiener process when time flows backwards from $T$ to 0. 
\end{itemize}
Diffusion models train a neural network $\epsilon_{\theta}(x_t,t)$ parameterized by $\theta$ with the following objective~\citep{ho2020DDPM,score2023songyang}:
\begin{equation}
\min_{\theta} \mathbb{E}_{x_t, \, t\sim\mathcal{U}([0,T]), \epsilon\sim\mathcal{N}(0, I)} \left[\|\epsilon-\epsilon_{\theta}(x_t,t)\|_2^2\right],
\end{equation}
where \( x_0 \sim q_0(x_0) \), \( x_t = \alpha_t x_0 + \sigma_t \epsilon \). The fitted $\theta$ can be use to estimate the score function $\nabla_{x_t}\log q_t(x_t)$ by $\epsilon_{\theta}(x_t,t) \approx - \sigma_t \nabla_{x_t}\log q_t(x_t)$ and substituted into Eq.~(\ref{eq:ode}-\ref{eq:sde}) for $x_0$ generation.

\textbf{Energy-guided diffusion sampling.} To bias sampling toward preferred distributions, the prior energy-guidance framework~\citep{CEP} provides the re-weighted distribution $p_0(x_0)$:
\begin{equation} \label{eq:energy-prop-origin}
p_0(x_0) \propto q_0(x_0) f(\mathcal{E}(x_0)),
\end{equation}
where \( q_0(x_0) \) is the unknown data distribution, \( \mathcal{E}(x_0) \) is any form of energy function that encodes human preferences, and \( f(x) \geq 0 \) can be any non-negative function. To solve this sampling problem, previous approaches train a separate time-dependent classifier $\mathcal{E}_t$~\citep{CEP,pmlr-v162-janner22a}, which is computationally expensive and may introduce additional training errors. To address this, previous work~\citep{Zheng2024fisor} introduces weighted regression method to perform exact energy guidance:
\begin{lemma}[\textit{Weighted regression as exact energy guidance}] \label{lemma:weight-regression-fisor}
We can sample $x_0 \sim p_0(x_0)$ where $p_0$ is the weighted distribution in Eq.~(\ref{eq:energy-prop-origin}) by (i) optimizing the weighted regression loss:
\begin{equation} \label{eq:weighted-regression-fisor-thm}
\min_{\theta} \mathbb{E}_{x_t, t\sim\mathcal{U}([0,T]), \epsilon\sim\mathcal{N}(0, I)} \left[f\big(\mathcal{E}(x_0)\big)\|\epsilon-\epsilon_{\theta}(x_t,t)\|_2^2\right],
\end{equation}
where \( x_0 \sim q_0(x_0) \), \( x_t = \alpha_t x_0 + \sigma_t \epsilon \); and (ii) substituting the obtained $\epsilon_\theta$ into diffusion ODEs/SDEs~\citep{song2021scorebased} solving process.
\end{lemma}

\textbf{Task-conditioned diffusion sampling.} In this work, we extend the Lemma~\ref{lemma:weight-regression-fisor} and provide the following proposition for task-specific sampling:
\begin{proposition}[Task-conditioned diffusion policy extraction via weighted regression]
\label{pro:weighted-regression}
The extraction of optimal policy $\pi_z^*$ in Eq.~(\ref{weight_bc}) can be achieved by (i) minimizing the weighted regression loss defined as:
\[
\min_{\theta} \mathbb{E}_{t\sim\mathcal{U}([0,T]), \epsilon\sim\mathcal{N}(0, I),(s,a)\sim \mathcal{D}} \Big[\exp\Big( \alpha \cdot(F(s,a, z)^\top z - V_{\pi_z}(s,z))\Big) \|\epsilon - \epsilon_{\theta,z}(a_t,s,z,t)\|_2^2 \Big],
\]
and (ii) solving diffusion ODEs/SDEs~\citep{song2021scorebased} by substituting $\epsilon_{\theta,z}$ in Eq.~(\ref{eq:ode}-\ref{eq:sde}). 
\end{proposition}

\begin{proof}
We aim to extract the optimal policy forming as a weighted distribution as follows:\[
\pi_z^*(a|s) \propto \mu(a|s) \exp\left( \alpha \cdot (F(s, a, z)^\top z - V_{\pi_z}(s,z)) \right),
\]
where $\mu(a|s)$ is the behavior distribution in dataset $\mathcal{D}$ and $F(s, a, z)^\top z - V_{\pi_z}(s,z)$ is the task-dependent weight function. Set \( p_z(a|s) \) as a task-conditioned reweighted version of the unknown data distribution \( q_0(a|s) \) and substitute the task-dependent weight function into Eq.~(\ref{eq:energy-prop-origin}), we have:
\begin{equation} \label{eq:eq:energy-prop-origin-task-condition }
    p_z(a|s) \propto q_0(a|s) \exp\left(\alpha \cdot \left(F(s,a,z)^\top z - V_{\pi_z}(s,z)\right)\right).
\end{equation}
There exist a normalization constant \( Z(s,z) = \int q_0(a|s) \exp\left(\alpha \cdot (F(s,a,z)^\top z - V_{\pi_z}(s,z))\right) da \) ensures \( p_z(a|s) \) to be a valid distribution. We now aim to sample from the distribution \( p_z(a|s) \).

Parameterize the denoising process with a task-conditioned model $\epsilon_{\theta,z}(a_t,s,z,t)$ and substitute in Eq.~(\ref{eq:weighted-regression-fisor-thm}), we can get a modified task-conditioned weighted regression objective for the diffusion models by replacing the energy function with $F(s,a,z)^\top z - V_{\pi_z}(s,z)$:
\begin{equation} \label{eq:diffusion_objective-derive}
\min_{\theta} \mathbb{E}_{t\sim\mathcal{U}([0,T]), \epsilon\sim\mathcal{N}(0, I),(s,a)\sim \mathcal{D}} \Big[\exp\Big( \alpha \cdot(F(s,a, z)^\top z - V_{\pi_z}(s,z))\Big) \|\epsilon - \epsilon_{\theta,z}(a_t,s,z,t)\|_2^2 \Big].
\end{equation}
According to the Lemma~\ref{lemma:weight-regression-fisor}, substituting \(\epsilon_{\theta,z}^*\) into the diffusion ODEs/SDEs \cite{song2021scorebased} generates data that sample from \(p_z(a|s)\), thereby recovering \(\pi_z^*(a|s)\). The proof is completed. 
\end{proof}
\section{Experimental Setups}
\label{appendix:Experimental_details}
Our experiments span two benchmarks, ExORL~\citep{yarats2022dontchangealgorithmchange} and D4RL~\citep{fu2021d4rldatasetsdeepdatadriven}, across two locomotion domains (\textit{Walker} and \textit{Quadruped}) and two manipulation goal-reaching domains (\textit{Jaco} and \textit{Kitchen}). All experiments we conducted are in state-based tasks.

\subsection{Datasets}
\label{appendix:datasets}
\paragraph{ExORL~\citep{yarats2022dontchangealgorithmchange}.} ExORL consists of datasets collected by several unsupervised RL algorithms \cite{laskin2021URLB} on the DeepMind Control Suite~\citep{tassa2018deepmindcontrolsuite}. We select datasets collected by four unsupervised RL algorithms: \textit{APS}~\citep{Liu2021APSAP}, \textit{RND}~\citep{Burda2018ExplorationBR}, \textit{PROTO}~\citep{Yarats2021ReinforcementLW}, and \textit{DIAYN}~\citep{Eysenbach2018DiversityIA} for each of the three domains (\textit{Walker}, \textit{Jaco}, and \textit{Quadruped}). Each domain has four tasks for evaluation at test time.

\paragraph{D4RL~\citep{fu2021d4rldatasetsdeepdatadriven} Kitchen.} D4RL Kitchen datasets consist of robotic arm manipulating trajectories of different tasks in the Franka Kitchen~\citep{lynch2019play} domain. We adopt two datasets: "kitchen-partial-v0" and "kitchen-mixed-v0". 
\begin{itemize}[leftmargin=*]
\item \textbf{Partial}: Subtasks involve the microwave, kettle, light switch, and slide cabinet. Tasks are guaranteed to be solved in a subset of the 'partial' dataset.
\item \textbf{Mixed}: Subtasks involve microwave, kettle, bottom burner, and light switch. No trajectories solve any tasks completely in the 'mixed' dataset.
\end{itemize}

\paragraph{Dataset preprocessing.} BREEZE incorporates standardized data normalization for diffusion models across all training datasets as a preprocessing step to stabilize the learning process and enhance generalization capability.

\subsection{Domains}
\label{appendix:domains}

\begin{figure}[t]
    \centering
    \includegraphics[width=0.9\linewidth]{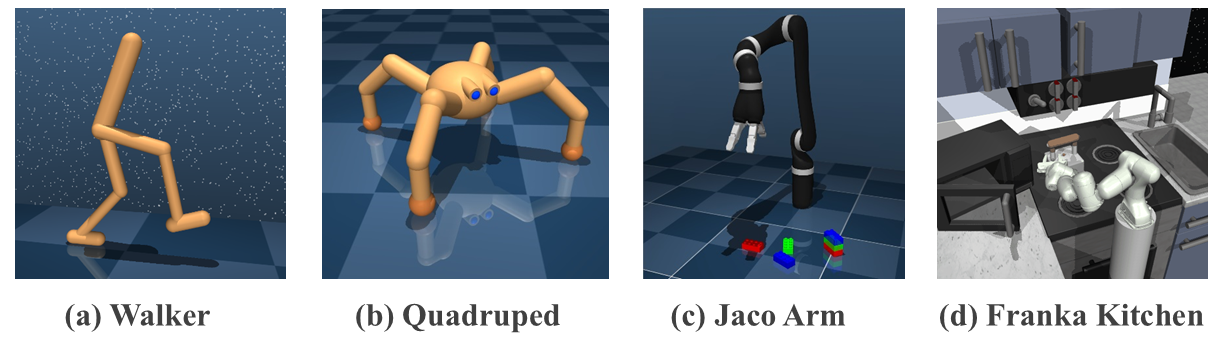}
    \caption{\textbf{Domains.} Walker, Quadruped, Jaco, and Franka Kitchen.}
    \label{fig:env}
\end{figure}

\begin{itemize}[leftmargin=*]
    \item \textbf{Walker} (Locomotion): A bipedal robot with 24-dimensional states (joint positions/velocities) and 6-dimensional actions. Test tasks include \textit{Flip}, \textit{Run}, \textit{Stand}, and \textit{Walk}. Rewards combine dense objectives: maintaining torso height (\textit{Stand}), achieving target velocities (\textit{Run/Walk}), or angular momentum (\textit{Flip}).

    \item \textbf{Quadruped} (Locomotion): A four-legged robot with 78-dimensional states and 12-dimensional actions. Tasks include \textit{Jump}, \textit{Run}, \textit{Stand}, and \textit{Walk}, with rewards for torso stability and velocity tracking. 

    \item \textbf{Jaco} (Manipulation/Goal-reaching): A 6-DoF robotic arm with 55-dimensional states and 6-dimensional actions. Tasks involve reaching four target positions (\textit{Top/Bottom Left/Right}) using sparse rewards based on proximity to goals. 

    \item \textbf{Franka Kitchen}~\citep{lynch2019play} (Manipulation/Goal-reaching): A robotic manipulation environment controlling a 9-DoF Franka robot with multiple given test-time objectives. The agent must sequentially achieve four subtasks per episode, receiving sparse rewards for each subtask. States include arm joint positions, velocities, torques, and task-specific object features. Test-time goals are defined as proprioceptive states concatenated with target object states.
\end{itemize}

\subsection{Baselines}
We utilize the following open-source codebases for experiments on baselines. We set the batch size to 512 to ensure a fair comparison with other methods, while keeping all other hyperparameters at their default values.

\begin{itemize} [leftmargin=*]
    \item For \textit{SF-LAP}~\citep{Borsa2018UniversalSF} and \textit{FB}~\citep{Touati2021FB}, \textit{MCFB} and \textit{VCFB}~\citep{Jeen2023ZeroShotRL}, we utilized the open-source implementation available at \url{https://github.com/enjeeneer/zero-shot-rl}.

    \item For \textit{HILP} \cite{park2024foundationpolicieshilbertrepresentations}, we utilized the open-source implementation for state-based zero-shot RL setting, which is available at \url{https://github.com/seohongpark/HILP}.
\end{itemize}

\subsection{Architectures}
\label{appendix:Architectures}
This section details the model architectures used in BREEZE.

\paragraph{$F$ Network.}
The $F$ network processes inputs $(s, a)$ and $(s, z)$ using two separate MLPs, each with two hidden layers of size 1024 and ReLU activation, projecting the inputs into a 512-dimensional feature space. The resulting embeddings are concatenated along the sequence dimension and processed through two identical processing blocks. Each block consists of a self-attention layer~\citep{vaswani2017attention}, followed by a feed-forward network with residual connections, LayerNorm~\citep{ba2016layer}, and dropout (rate=0.1). Specifically, the final representation is flattened and passed through two independent output heads, denoted as $F1$ and $F2$, each comprising two linear layers with a hidden dimension of 256 that map to the $d$-dimensional output space.

\paragraph{$B$ Network.}
The $B$ network is based on a standard transformer architecture with 8-head multi-head attention~\citep{vaswani2017attention} and ReLU activation. State embeddings are processed through transformer blocks and projected to a $d$-dimensional space via a linear layer and further scaled to $\sqrt{d}$ by L2 normalization. We observe that a sufficiently large $B$ network is essential for achieving strong performance on the ExORL benchmark.

\paragraph{Diffusion Policy.}
The diffusion policy uses a residual-connected MLP as the noise predictor, following the IDQL~\citep{idql} implementation. The hidden dimension is set to 1024. The task vector $z$ is incorporated via Feature-wise Linear Modulation (FiLM)~\citep{perez2018film}: two parallel linear layers generate a scaling factor $\gamma$ and a shifting factor $\beta$ from $z$, which modulate intermediate features $h$ as $\gamma \odot h + \beta$.

\subsection{Sampling of \( z \)}
\label{subsec:z_sampling}

Following previous studies~\citep{Touati2022DoesZR,Jeen2023ZeroShotRL}, we utilize a mixture of two methods for sampling the latent variable \( z \):
\begin{itemize} %[leftmargin=*]
    \item Sample \( z \) uniformly within a sphere of radius \( \sqrt{d} \) in \( \mathbb{R}^d \).
    
    \item Derive \( z \) by setting \( z = B(s) \), where $s \sim \mathcal{D}$ is randomly sampled from the data set.
\end{itemize}

In our primary experiments, we employ a balanced mix ratio of 50\% between the two methods. 

\subsection{Value Learning}
\label{subsec:optimal_sampling}
\paragraph{Training details.} We implement the $V$-network with 3-layer MLPs with 512 hidden dimensions and ReLU activation functions. During training, we set the $\tau$ for expectile regression in Eq.~(\ref{eql:vloss}) shown in Section ~\ref{appendix:Hyperparameter}. We use clipped double value learning~\citep{fujimoto2018addressing} for both the $Q$-value and the $M$-value.

\paragraph{Rejection sampling.} To improve the zero-shot performance, we use rejection sampling rather than tuning existing parameters for greedy optimization. Rejection sampling is commonly used in diffusion policies~\citep{Chen2022OfflineRL,idql,he2024aligniql}, a stable mechanism that selects $N$ actions from the policy to boost with the highest value. Our practical implementation of rejection sampling is in two phases:

\begin{itemize}[leftmargin=*]
    \item \textbf{Training phase}:
    To compute the FB loss (Eq.~\ref{fb_loss}),  we first sample a full transition batch, which contains states, actions, next states, and next actions. A hyperparameter $\rho_a$, referred to as the dataset mixture ratio, determines the proportion of next actions drawn directly from the dataset. The remaining fraction is replaced by actions sampled from the policy $\pi_z$: we generate $K_\text{train}$ candidate actions ${a_{t+1}^{(1)}, \dots, a_{t+1}^{(K_\text{train})}}$ via the diffusion policy and select the optimal next action according to:
    \begin{equation}
    a_{t+1}^* = \arg\max_{a_{t+1}^{(i)}} ; F(s_{t+1}, a_{t+1}^{(i)}, z)^\top z .
    \end{equation}

    \item \textbf{Evaluation phase}: We sample $K_\text{eval}$ candidate actions during evaluation, and select the optimal action according to the $Q$-value.
\end{itemize}
Details of practical $\rho_a$, $K_\text{train}$ and $K_\text{eval}$ is provided in Section \ref{appendix:Hyperparameter}.

\subsection{Hyperparameter} 
\label{appendix:Hyperparameter}
This section provides the detailed hyperparameter setup. In our experiments, the model architecture and basic algorithm hyperparameters remain unchanged, as detailed in Table \ref{table:general_hyp}. Domain-specific hyperparameters are detailed in Table \ref{table:hyp_domain_full} and Table \ref{table:hyp_domain_small}.

\begin{table}[ht]
\caption{General hyperparameters used for BREEZE}
\label{table:general_hyp}
\centering
% \resizebox{0.8\columnwidth}{!}{
\begin{tabular}{@{}lll@{}}
\toprule
\multicolumn{1}{c}{\multirow{17}{*}{\centering General hyperparameters}} & Hyperparameter                    & Value     \\ \cmidrule(lr){2-3}
                                                                    & Optimizer                         & AdamW      \\
                                                                    & Representation learning rate      & 1e-4      \\
                                                                    & $V$ network learning rate         & 3e-4      \\
                                                                    & Policy learning rate              & 1e-4      \\
                                                                    & Discount factor $\gamma$          & 0.98      \\
                                                                    & Learning steps                    & 1,000,000 \\
                                                                    & Mini-batch                        & 512       \\
                                                                    & Representation soft update factor $\lambda$    & 0.01      \\
                                                                    & Policy soft update factor $\lambda$   & 0.001      \\
                                                                    & Latent dimension $d$               & 50 \\
                                                                    & $z$ mixing ratio                  & 0.5 \\
                                                                    & $z$ inference steps               & 10,000 \\
                                                                    & Regularization weight for orthogonality loss $w_1$ & 10 \\
                                                                    & Beta schedule                     & vp        \\
                                                                    & Exponential advantages clip & $(-\infty, 100]$\\
                                                                    \cmidrule(lr){1-3}
\multicolumn{1}{c}{\multirow{21}{*}{\centering Architecture}}       & \(F\) preprocessor hidden dimension      & 1024       \\
                                                                    & \(F\) preprocessor hidden layers    & 2     \\
                                                                    & \(F\) preprocessor output dimension    & 512     \\
                                                                    & \(F\) preprocessor activation function   & relu    \\
                                                                    & \(F\) attention blocks              & 2     \\
                                                                    & \(F\) linear layer hidden dimension        & 256     \\
                                                                    & \(F\) dropout          & 0.1     \\
                                                                    & \(B\) nhead                          & 8       \\
                                                                    & \(B\) encoder layers              & 2       \\
                                                                    & \(B\) decoder layers              & 2       \\
                                                                    & \(B\) d model             & 256    \\
                                                                    & \(B\) dropout             & 0.1    \\
                                                                    & \(B\) feedforward dimension             & 2048    \\
                                                                    & \(B\) activation function           & relu   \\
                                                                    & \(V\) hidden dimension           & 512   \\
                                                                    & \(V\) hidden layers           & 2   \\
                                                                    & \(V\) activation function          & relu   \\
                                                                    & Policy MLP blocks                   & 3       \\
                                                                    & Policy hidden dimension             & 1024    \\
                                                                    & Policy activation function          & mish    \\
                                                                    & Policy time embedding               & learned    \\
                                                                    \bottomrule
\end{tabular}
% }
\end{table}

\begin{table}[ht]
\centering
\caption{\label{table:hyp_domain_full} Domain-specific hyperparameters for BREEZE with full dataset.}
\resizebox{\linewidth}{!}{\scriptsize
\begin{tabular}{lccccccccc}
\toprule
\textbf{Domain-dataset} & $K_{\text{train}}$   & $\rho_a$  & $K_{\text{eval}}$    & expectile $\tau$    & $F\text{-reg}$ coef. $\omega_q$ & Diffusion steps $T$ & Temperature $\alpha$ \\
\midrule
Walker-RND              & 9 & 0.1 & 64 & 0.99 & 0.001 & 5 & 0.05 \\
Walker-APS              & 9 & 0.1 & 64 & 0.99 & 0.001 & 5 & 0.05\\
Walker-PROTO            & 9 & 0.1 & 64 & 0.99 & 0.001 & 5 & 0.05\\
Walker-DIAYN            & 9 & 0.1 & 64 & 0.99 & 0.001 & 5 & 0.05\\
\cmidrule(lr){1-8}
Jaco-RND                & 1 & 0 & 64 & 0.99 & 0.001 & 10 & 0.05\\
Jaco-APS                & 9 & 0.1 & 64 & 0.99 & 0.0001 & 5 & 0.05\\
Jaco-PROTO              & 1 & 0 & 64 & 0.99 & 0.001 & 10 & 0.05\\
Jaco-DIAYN              & 9 & 0.1 & 64 & 0.99 & 0.0001 & 5 & 0.05\\
\cmidrule(lr){1-8}
Quadruped-RND           & 1 & 0 & 64 & 0.99 & 0.001 & 10 & 0.05\\
Quadruped-APS           & 1 & 0 & 64 & 0.99 & 0.001 & 10 & 0.05\\
Quadruped-PROTO         & 1 & 0 & 64 & 0.99 & 0.001 & 10 & 0.05\\
Quadruped-DIAYN         & 1 & 0 & 64 & 0.99 & 0.001 & 10 & 0.05\\
\cmidrule(lr){1-8}
Kitchen-mixed           & 1 & 0 & 16 & 0.7 & 0.001 & 5 & 0.1\\
Kitchen-partial         & 2 & 0.2 & 4 & 0.7 & 0.001 & 5 & 0.08\\
\bottomrule
\end{tabular}}
\end{table}

\begin{table}[ht]
\centering
\caption{\label{table:hyp_domain_small} Domain-specific hyperparameters for BREEZE with small sample dataset.}
% \vspace{-1em}
\resizebox{\linewidth}{!}{\scriptsize
\begin{tabular}{lccccccccc}
\toprule
\textbf{Domain-dataset} & $K_{\text{train}}$   & $\rho_a$  & $K_{\text{eval}}$    & expectile $\tau$    & $F\text{-reg}$ coef. $\omega_q$ & Diffusion steps $T$  & Temperature $\alpha$ \\
\midrule
Walker-RND              & 2 & 0.2 & 32 & 0.99 & 0.0001 & 10 & 0.05\\
Walker-APS              & 2 & 0.2 & 32 & 0.99 & 0.0001 & 10 & 0.05\\
Walker-PROTO            & 2 & 0.2 & 32 & 0.99 & 0.001 & 10 & 0.05\\
Walker-DIAYN            & 8 & 0.2 & 32 & 0.99 & 0.001 & 10 & 0.05\\
\cmidrule(lr){1-8}
Jaco-RND                & 1 & 0 & 32 & 0.99 & 0.0001 & 10 & 0.05\\
Jaco-APS                & 1 & 0 & 32 & 0.99 & 0.0001 & 10 & 0.05\\
Jaco-PROTO              & 1 & 0 & 32 & 0.99 & 0.0001 & 10 & 0.05\\
Jaco-DIAYN              & 1 & 0 & 32 & 0.99 & 0.0001 & 10 & 0.05\\
\cmidrule(lr){1-8}
Quadruped-RND           & 2 & 0.2 & 32 & 0.99 & 0.001 & 10 & 0.05\\
Quadruped-APS           & 1 & 0 & 32 & 0.99 & 0.001 & 10 & 0.05\\
Quadruped-PROTO         & 2 & 0.2 & 32 & 0.99 & 0.001 & 10 & 0.05\\
Quadruped-DIAYN         & 2 & 0.2 & 32 & 0.99 & 0.0001 & 10 & 0.05\\
\bottomrule
\end{tabular}}
% \vspace{-1.2em}
\end{table}
\section{Pseudo-Code}
\label{appendix:Pseudo-Code}
Algorithm~\ref{alg:BREEZE} provides the pseudocode. Our implementation uses PyTorch~\cite{pytorch2019}, with all experiments conducted on a single NVIDIA A6000 GPU. Peak memory usage under 23 GB.

\begin{algorithm}[H]
\caption{Behavior-Regularized Zero-shot RL (BREEZE)}
\label{alg:BREEZE}
\begin{algorithmic}[1]
\Require Latent dimension $d$, mini-batch $b$, gradient steps $N$, orthogonality coefficient $w_1$, value regularization coefficient $w_q$.

\State Normalize dataset $\mathcal{D}$; Initialize networks $F_{\phi}$, $B_{\psi}$, $V_{\omega}$, diffusion actor $\pi_z$ with network $\epsilon_{\theta}$
\For{$n=1,\ldots,N$}
        \State Sample a mini-batch of transitions $\{(s_{t},a_{t},s_{t+1},a_{t+1})_i\}_{i\in |b|}\subset\mathcal{D}$
        \State Sample a mini-batch of $\{z_{i}\}_{i\in |b|}\sim\mathbb{R}^d$
        \State Collect $\{a'_{t+1}\}_{i\in |b|}$ by mixture of data batch and $\pi_z$ outputs
        \State \textcolor{blue}{// Representation learning:}
        \State Compute $\mathcal{L}_{\text{FB}}(\phi, \psi)$ and $\mathcal{L}_{F\text{-reg}}(\phi)$ using Eq.~(\ref{fb_loss}), Eq.~(\ref{eq:constraint})
        \State Compute orthogonality regularization loss $\mathcal{L}_{\text{ortho}}$
        \State $\mathcal{L}_{\mathrm{ortho}}(\psi) = \frac{1}{|b|^2}\sum_{i,j\in|b|}\left[(B_\psi(s_i)^{\top}B_\psi(s_j'))^{2}-\|B_\psi(s_i)\|_{2}^{2}-\|B_\psi(s_j')\|_{2}^{2}\right]$
        \State Update $F_{\phi}$ and $B_{\psi}$ by $\mathcal{L}_{\text{FB}} + w_1 \cdot \mathcal{L}_{\text{ortho}} + w_q \cdot \mathcal{L}_{F\text{-reg}}$
        \State \textcolor{blue}{// Value function learning:}
        \State Update $V_{\omega}$ using Eq.~(\ref{eql:vloss})
        \State \textcolor{blue}{// Diffusion policy learning:}
        \State Update ${\epsilon}_\theta$ using Eq.~(\ref{eq:diffusion_objective})
\EndFor
\end{algorithmic}
\end{algorithm}
\section{Additional Results}
This section provides a comprehensive showcase of all experimental results. 
\begin{itemize} %[leftmargin=*]
    \item Section \ref{appendix:curves} provides the learning curve comparisons.
    \item Section \ref{appendix:additional_result} reports the detailed results on ExORL benchmark.
    \item Section \ref{app:empirical_motivation} presents the visualizations of our empirically evaluated value distributions.
\end{itemize}

\subsection{Learning Curves}
\label{appendix:curves}
This section presents the zero-shot learning curves on the ExORL benchmark. Results on the full dataset and the 100k small-sample dataset are shown in Figure~\ref{fig:cur_quadruped}-\ref{fig:cur_walker} and Figure~\ref{fig:cur_quadruped100}-\ref{fig:cur_walker100}, respectively. The solid lines represent the mean IQM returns over 5 random seeds, and the shaded regions correspond to the standard deviation.

\begin{figure}[H]
\centering
\includegraphics[width=0.97\linewidth]{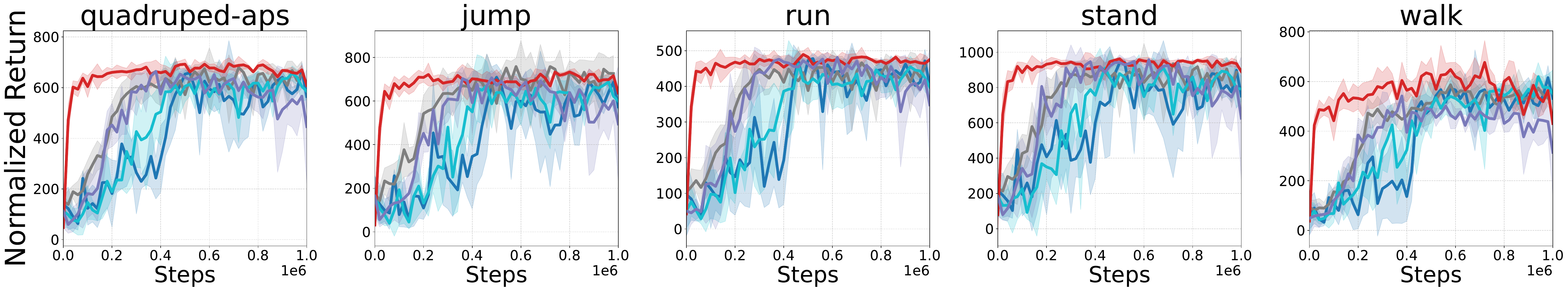}
\includegraphics[width=0.97\linewidth]{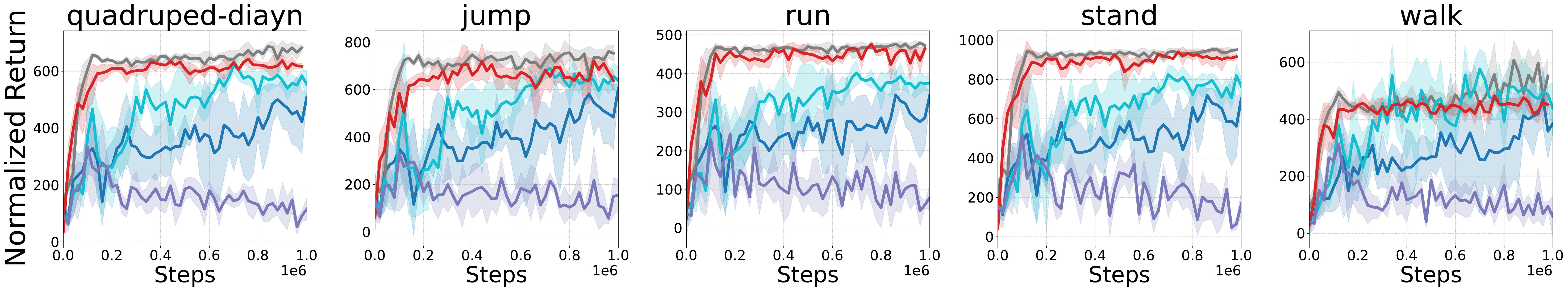}
\includegraphics[width=0.97\linewidth]{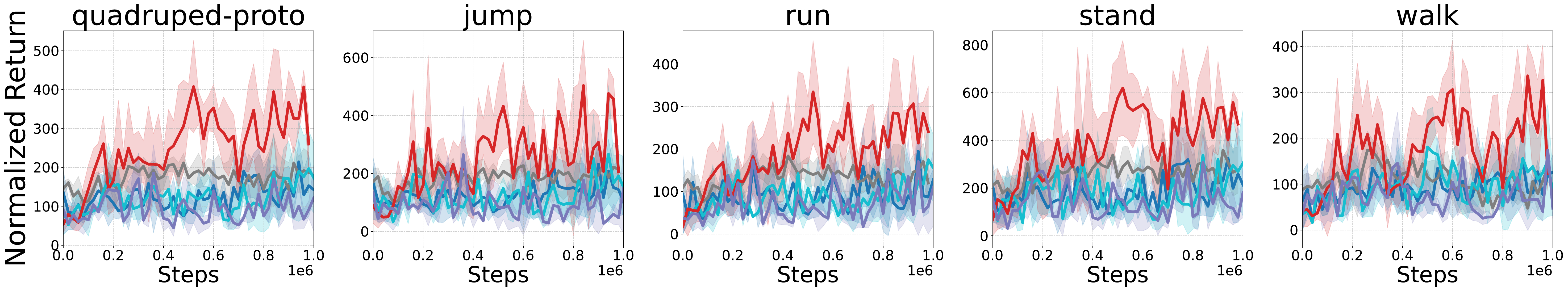}
\includegraphics[width=0.97\linewidth]{img/curves/quadruped-rnd.pdf}
\includegraphics[width=0.9\linewidth]{img/curves/legend.pdf}
\captionsetup{skip=3pt}
\caption{\label{fig:cur_quadruped}\small Curves of zero-shot performance on Quadruped domain.}
\vspace{-15pt}
\end{figure}

\begin{figure}[H]
\centering
\includegraphics[width=0.97\linewidth]{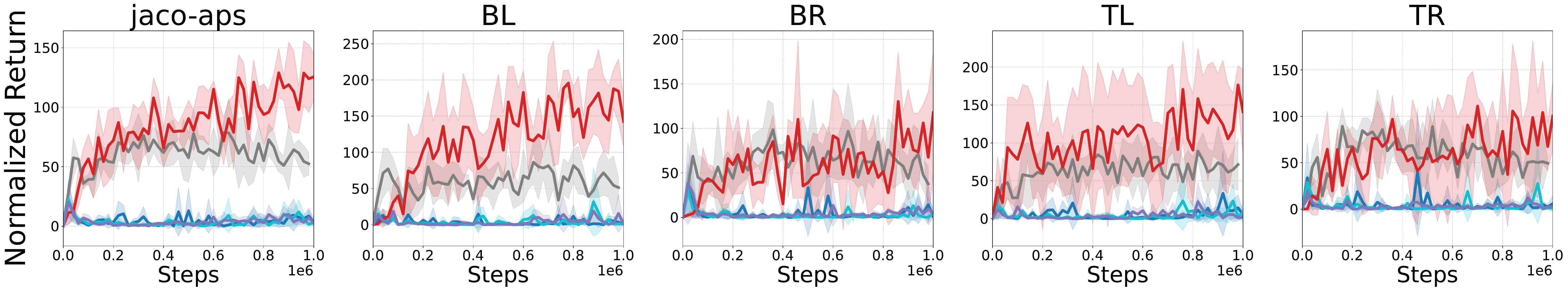}
\includegraphics[width=0.97\linewidth]{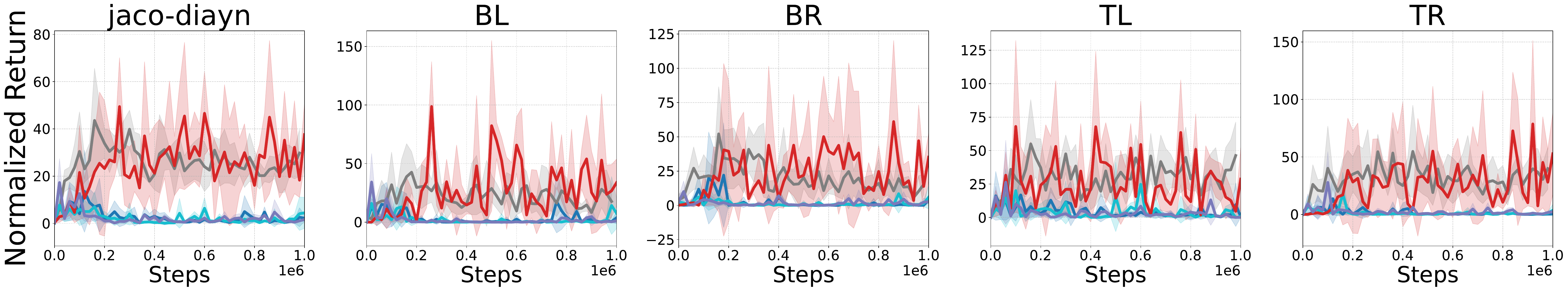}
\includegraphics[width=0.97\linewidth]{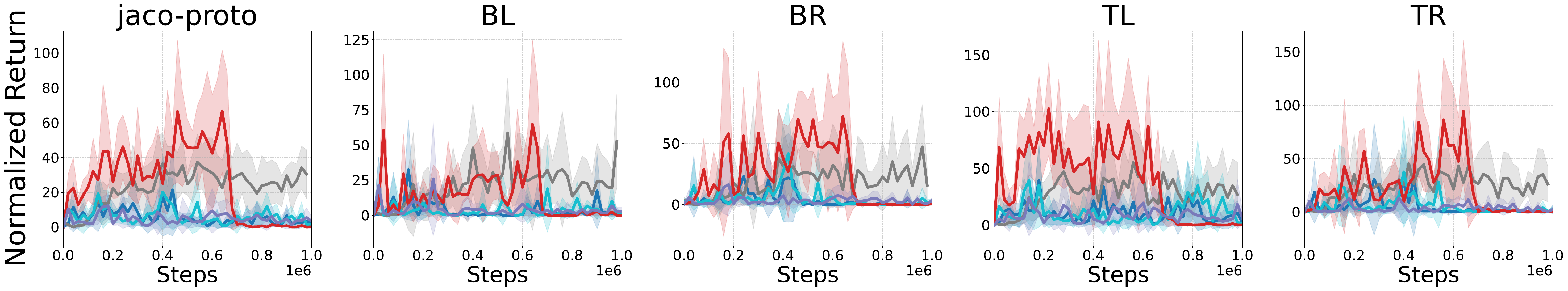}
\includegraphics[width=0.97\linewidth]{img/curves/jaco-rnd.pdf}
\includegraphics[width=0.9\linewidth]{img/curves/legend.pdf}
\captionsetup{skip=3pt}
\caption{\label{fig:cur_jaco}\small Curves of zero-shot performance on Jaco domain.} 
\vspace{-15pt}
\end{figure}

\begin{figure}[H]
% \vspace{-10pt}
\centering
\includegraphics[width=0.97\linewidth]{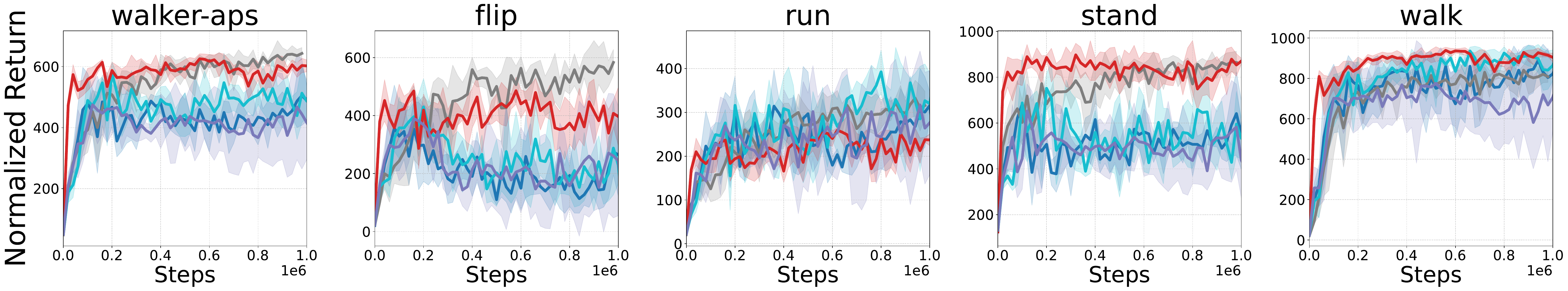}
\includegraphics[width=0.97\linewidth]{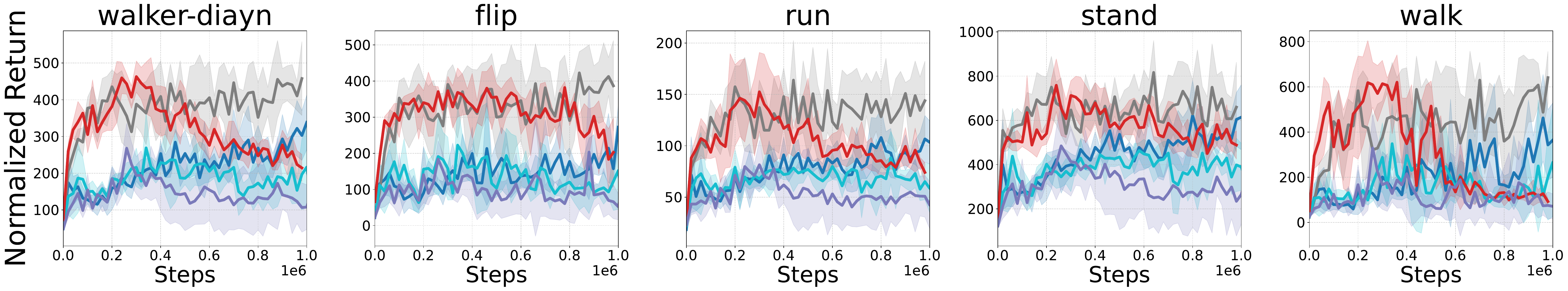}
\includegraphics[width=0.97\linewidth]{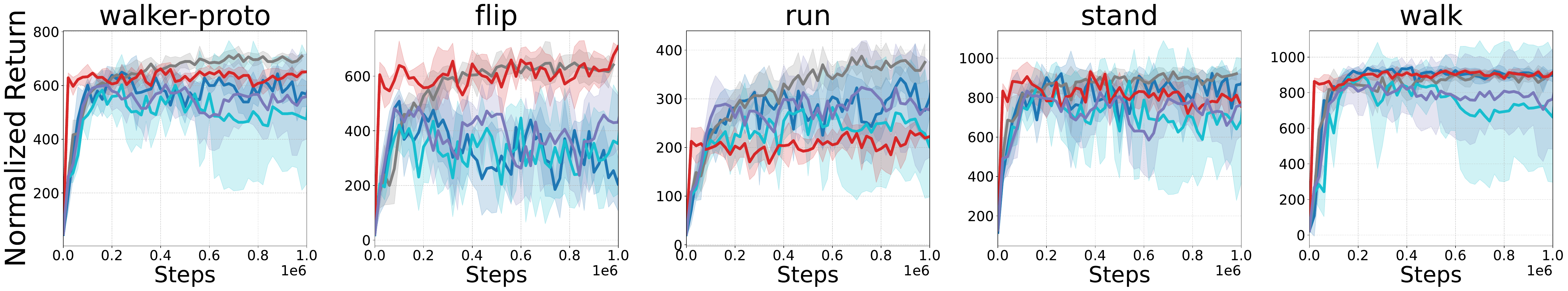}
\includegraphics[width=0.97\linewidth]{img/curves/walker-rnd.pdf}
\includegraphics[width=0.9\linewidth]{img/curves/legend.pdf}
\captionsetup{skip=3pt}
\caption{\label{fig:cur_walker}\small Curves of zero-shot performance on Walker domain.} 
\vspace{-15pt}
\end{figure}

\begin{figure}[H]
\centering
\includegraphics[width=0.97\linewidth]{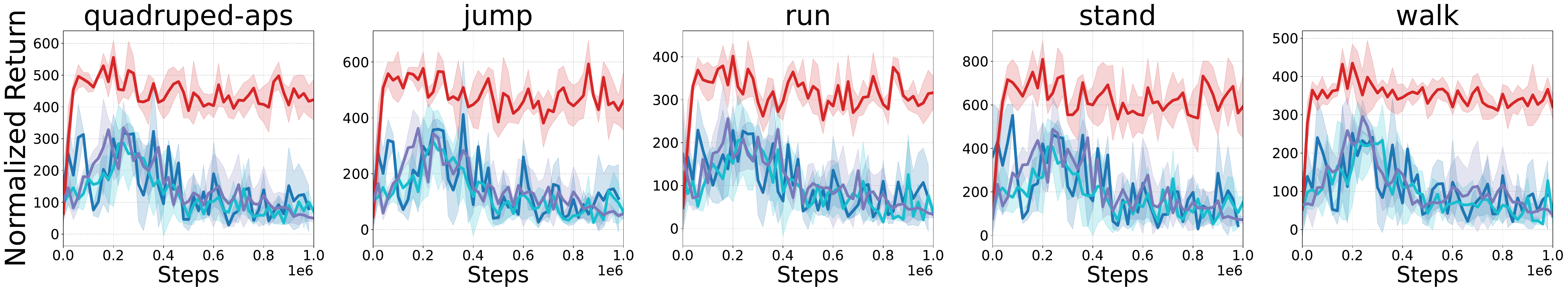}
\includegraphics[width=0.97\linewidth]{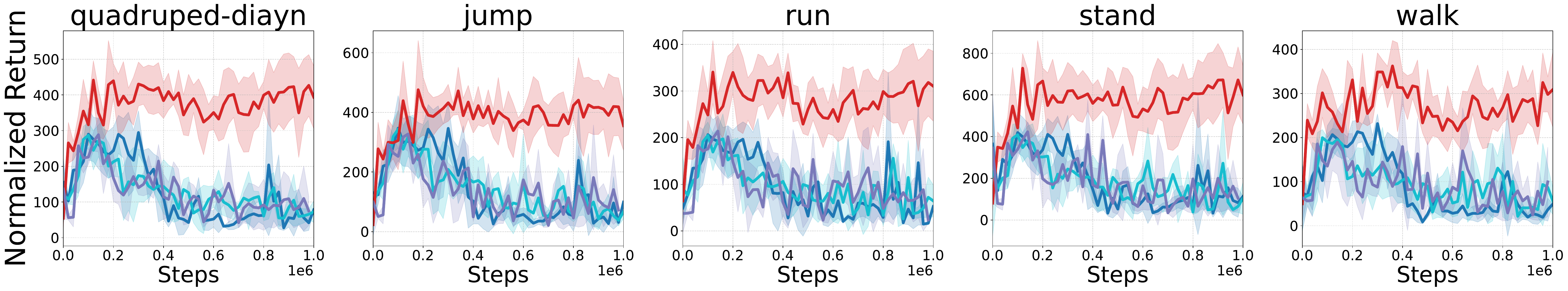}
\includegraphics[width=0.97\linewidth]{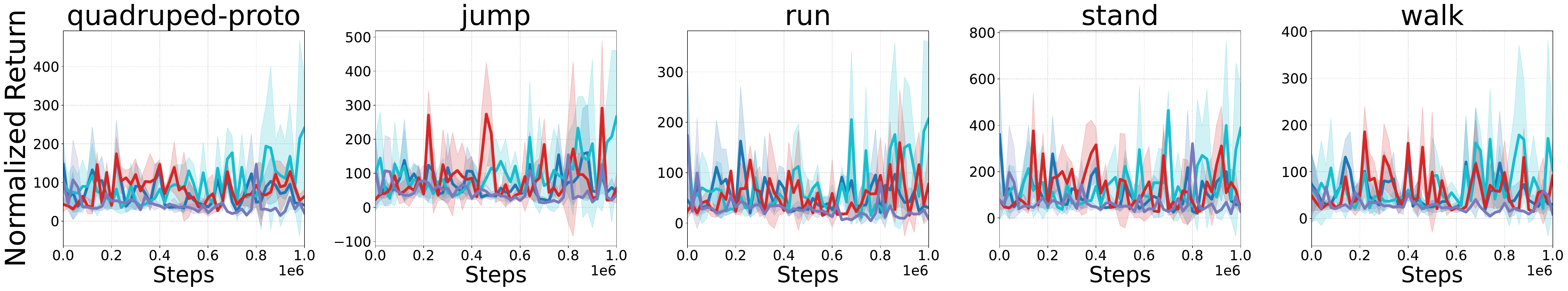}
\includegraphics[width=0.97\linewidth]{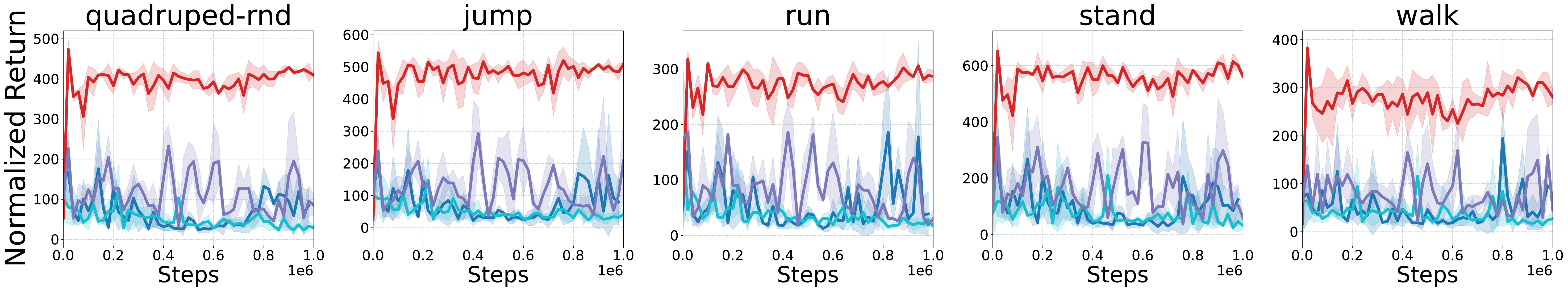}
\includegraphics[width=0.9\linewidth]{img/curves_small/legend.pdf}
\captionsetup{skip=3pt}
\caption{\label{fig:cur_quadruped100}\small Curves of zero-shot performance on Quadruped domain with 100k dataset.} 
\vspace{-15pt}
\end{figure}

\begin{figure}[H]
\centering
\includegraphics[width=0.97\linewidth]{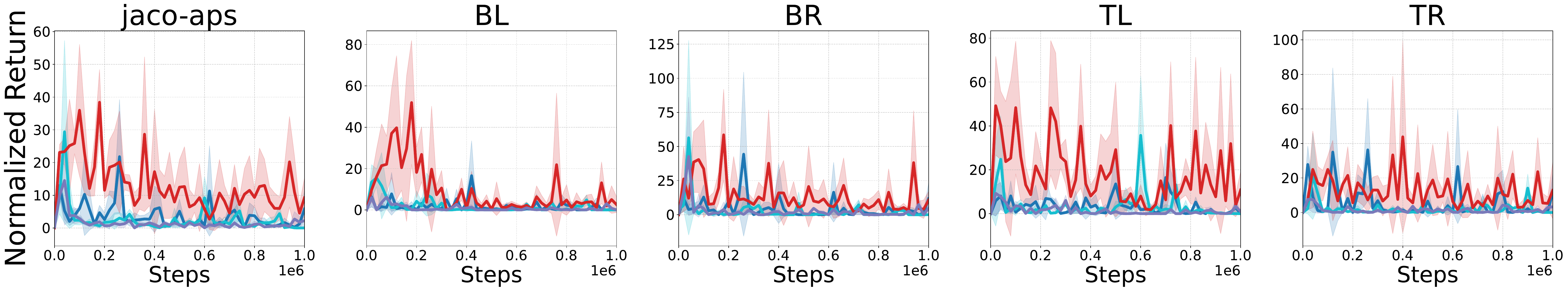}
\includegraphics[width=0.97\linewidth]{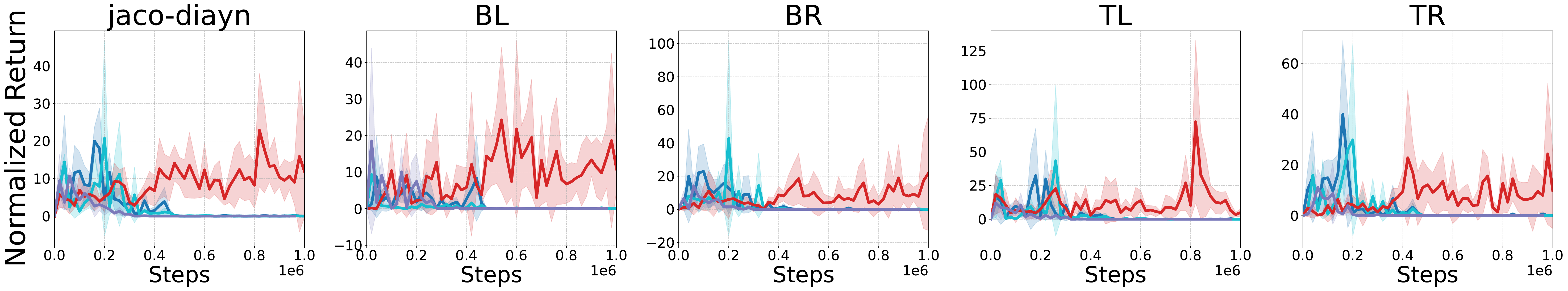}
\includegraphics[width=0.97\linewidth]{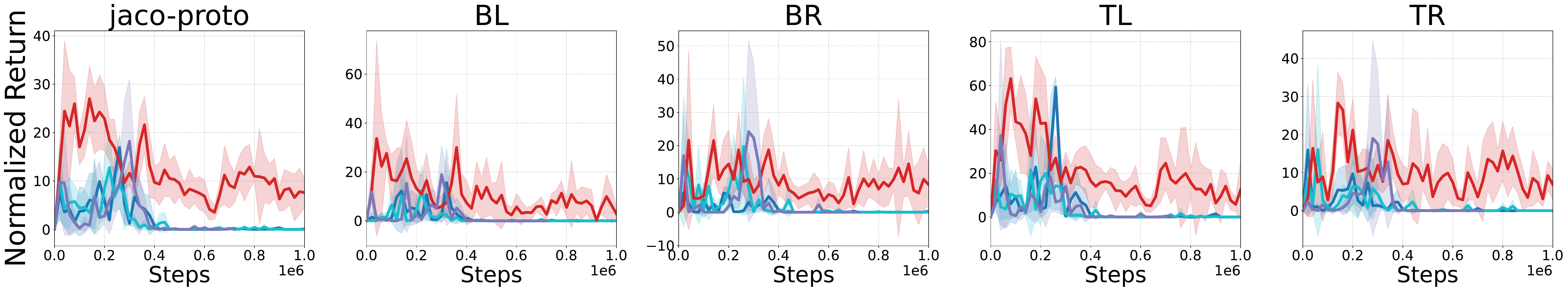}
\includegraphics[width=0.97\linewidth]{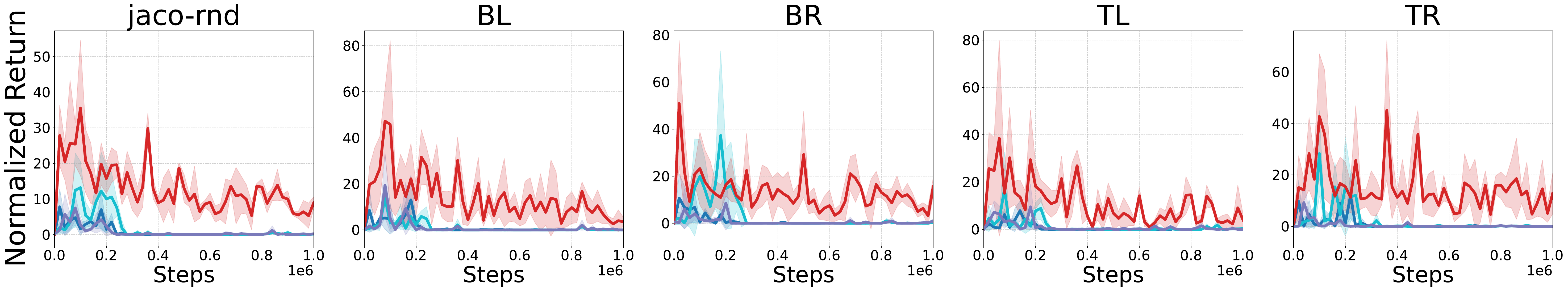}
\includegraphics[width=0.9\linewidth]{img/curves_small/legend.pdf}
\captionsetup{skip=3pt}
\caption{\label{fig:cur_jaco100}\small Curves of zero-shot performance on Jaco domain with 100k dataset.} 
\vspace{-15pt}
\end{figure}

\begin{figure}[H]
\centering
\includegraphics[width=0.97\linewidth]{img/curves_small/walker-aps.pdf}
\includegraphics[width=0.97\linewidth]{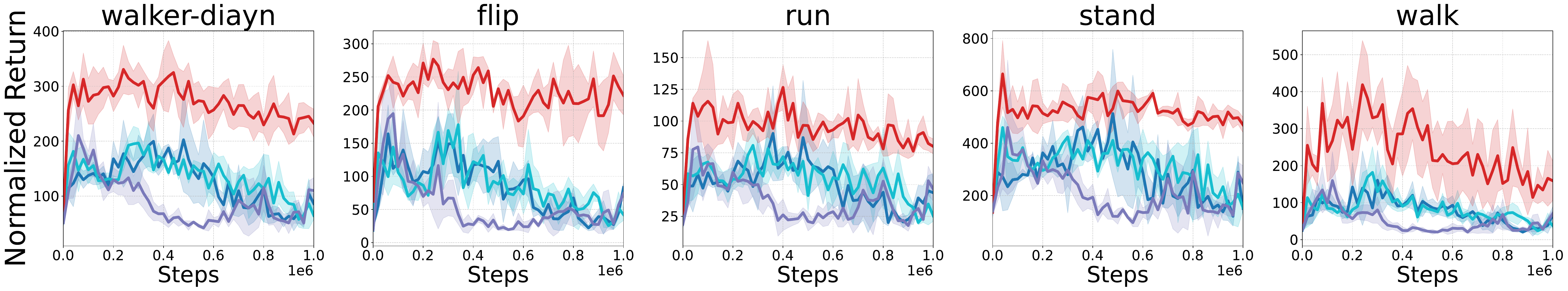}
\includegraphics[width=0.97\linewidth]{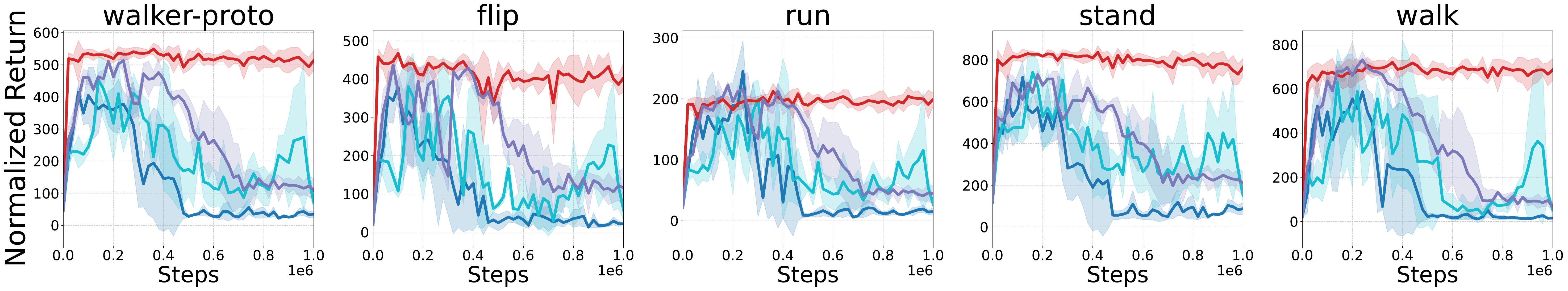}
\includegraphics[width=0.97\linewidth]{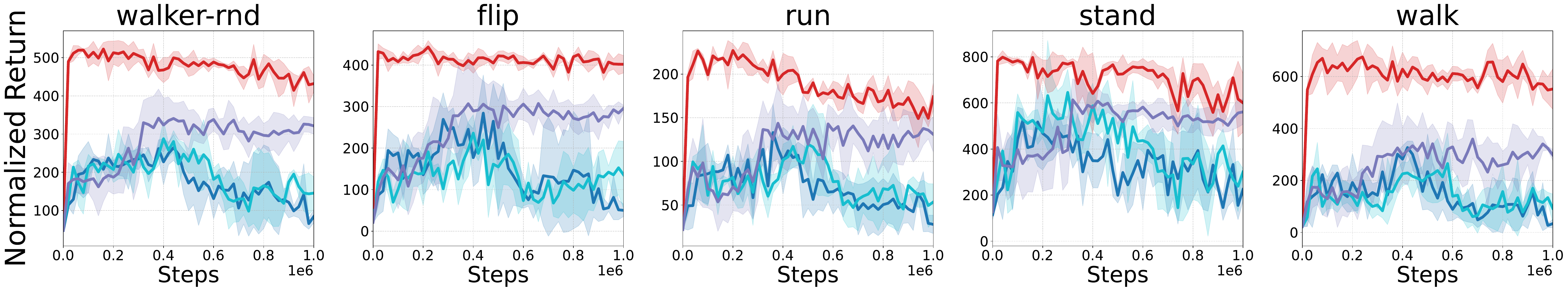}
\includegraphics[width=0.9\linewidth]{img/curves_small/legend.pdf}
\captionsetup{skip=3pt}
\caption{\label{fig:cur_walker100}\small Curves of zero-shot performance on Walker domain with 100k dataset.} 
\vspace{-15pt}
\end{figure}

\begin{figure}[H]
\centering
\includegraphics[width=0.24\linewidth]{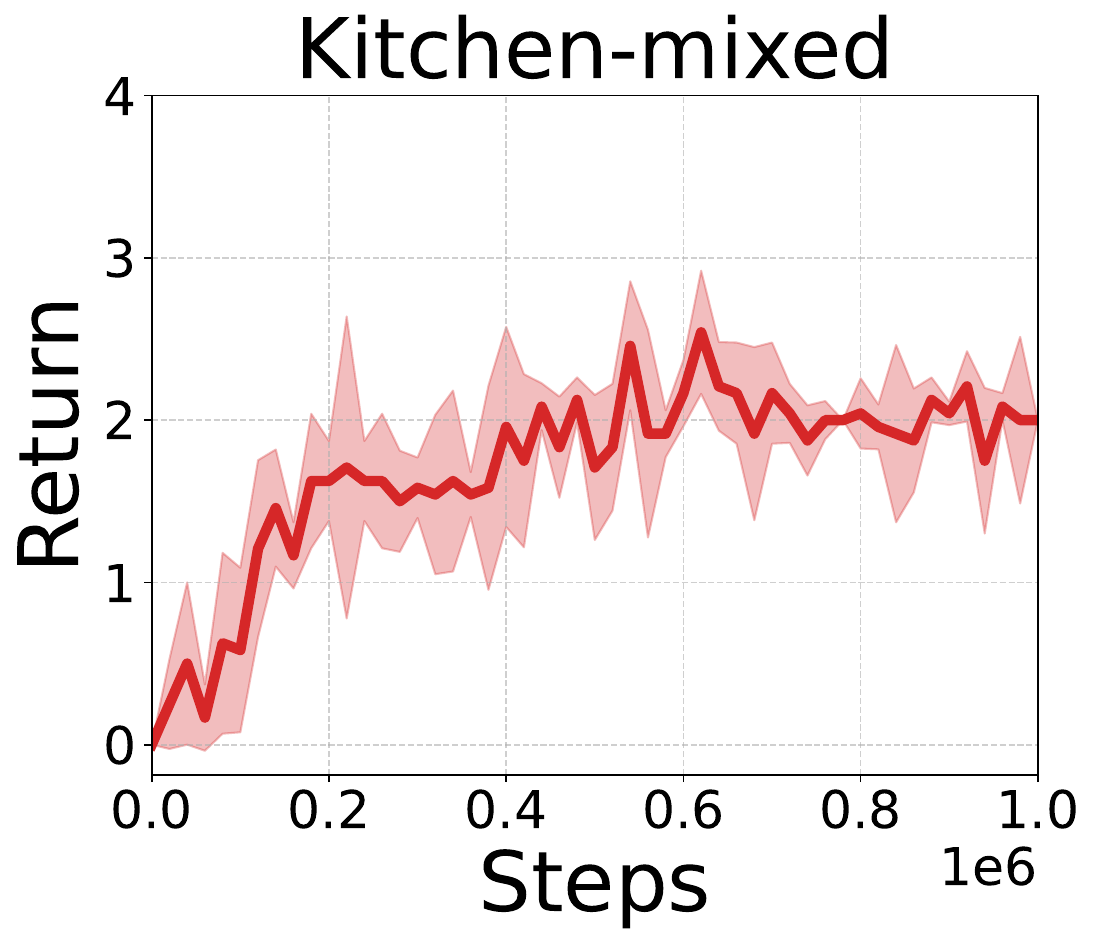}
\includegraphics[width=0.24\linewidth]{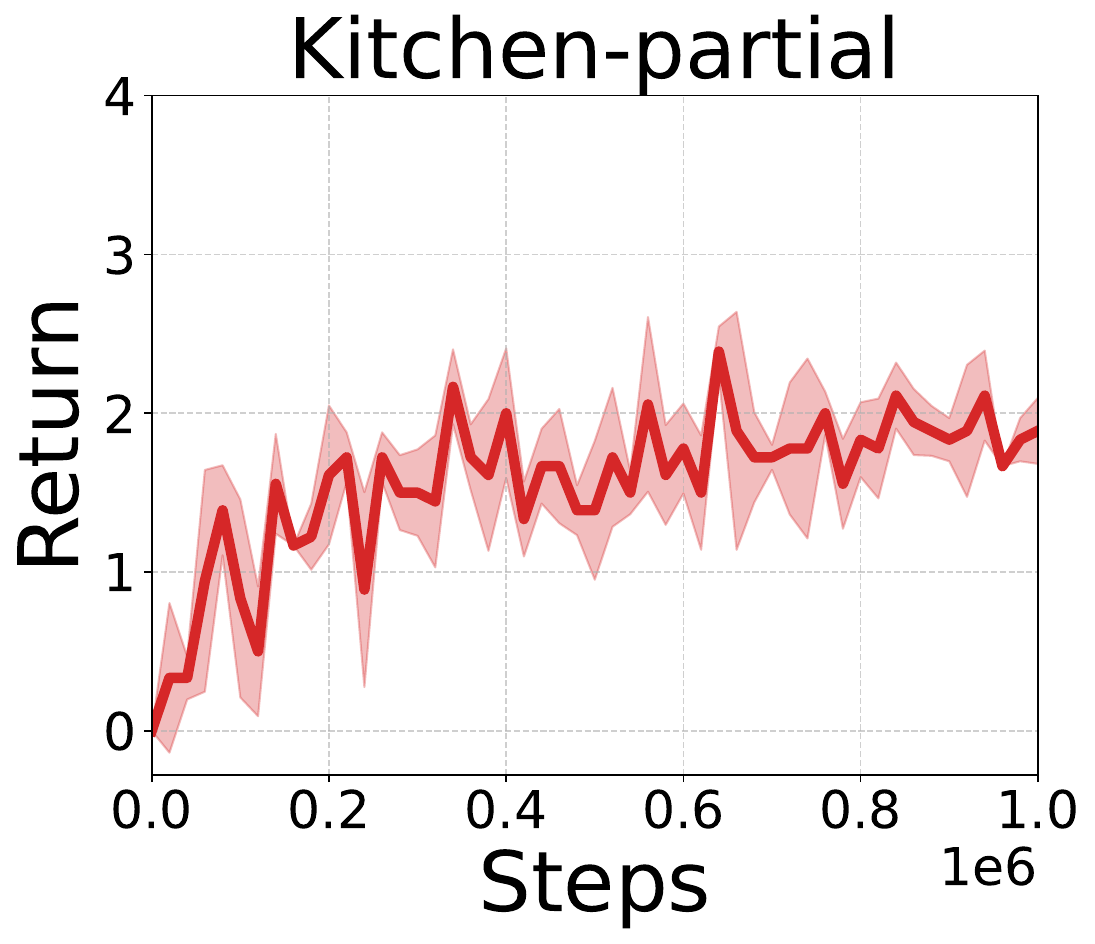}
\captionsetup{skip=3pt}
\caption{\label{fig:cur_kitchen}\small Curves of zero-shot performance on Kitchen domain.} 
\end{figure}

\subsection{Detailed Results on ExORL Benchmarks}
\label{appendix:additional_result}
This section presents complete experimental results on the ExORL benchmark. Table \ref{tab:fullresult} summarizes performance on the full dataset, while Table~\ref{tab:smallsampleresult} reports results on the 100k small-sample datasets. Evaluations were conducted every 10,000 steps during offline training. We computed the aggregate score across all tasks as the overall performance measure, and selected the highest domain-level score across 5 random seeds along with corresponding per-task results.

\newpage
\begin{table}[H]
\vspace{-15pt}
\centering
\caption{\small Full dataset experimental results on ExORL}
\resizebox{0.89\linewidth}{!}{\scriptsize
\begin{tabular}{lcccccccc}
\toprule
\textbf{Dataset} & \textbf{Domain} & \textbf{Task} & \textbf{SF-LAP} & \textbf{FB} & \textbf{VCFB} & \textbf{MCFB} & \textbf{HILP}  & \textbf{BREEZE}\\ 
\midrule
\multirow{16}{*}{RND} & \multirow{6}{*}{Walker}  & walk    & $555 \pm 221$  & \colorbox{mine}{$918 \pm 26$}  & $813 \pm 55$ & $842 \pm 126$ & $798 \pm 29$  & $907 \pm 23$ \\
                      &                          & stand   & $810 \pm 47$  & $834 \pm 39$  & $913 \pm 29$  & $950 \pm 20$  & $888 \pm 55$  & \colorbox{mine}{$938 \pm 12$} \\
                      &                          & run     & $235 \pm 61$  & $357 \pm 20$  & \colorbox{mine}{$376 \pm 22$} & $328 \pm 20$  & $368 \pm 34$  & $320 \pm 25$ \\
                      &                          & flip    & $465 \pm 107$  & $535 \pm 29$  & $511 \pm 52$  & $517 \pm 89$  & \colorbox{mine}{$608 \pm 48$}  & $606 \pm 22$ \\
\cmidrule(lr){3-9}
                      &                          & whole performance & $516 \pm 97$ & $661 \pm 10$  & $653 \pm 22$   & $659 \pm 51$ & $665 \pm 33$  & \colorbox{mine}{$693 \pm 16$} \\
\cmidrule(lr){2-9}
                    & \multirow{6}{*}{Jaco}  & reach top right     & $2 \pm 2$  & $75 \pm 68$   & $78\pm 63$    & $86 \pm 68$   & $43 \pm 42$   & \colorbox{mine}{$108 \pm 42$} \\
                    &                        & reach top left      & $2 \pm 2$  & $4 \pm 8$     & $15\pm 16$    & $10 \pm 9$    & $53 \pm 17$   & \colorbox{mine}{$68 \pm 30$} \\
                    &                        & reach bottom right  & $1 \pm 0$  & $46 \pm 42$   & \colorbox{mine}{$81 \pm 62$}   & $56 \pm 56$   & $53 \pm 41$   & $77 \pm 45$ \\
                    &                        & reach bottom left   & $70 \pm 71$  & $4 \pm 5$   & $9\pm 10$     & $11 \pm 13$   & $60 \pm 39$   & \colorbox{mine}{$84 \pm 55$} \\
\cmidrule(lr){3-9}
                    &                        & whole performance   & $18 \pm 18$  & $32 \pm 23$   & $46\pm 35$  & $41 \pm 34$   & $52 \pm 21$   & \colorbox{mine}{$84 \pm 14$} \\
                    
\cmidrule(lr){2-9}
                    & \multirow{6}{*}{Quadruped}  & walk   & $206 \pm 136$  & $519 \pm 35$  & $462 \pm 6$   & $603 \pm 68$  & $629 \pm 69$  & \colorbox{mine}{$641 \pm 64$} \\
                    &                             & stand  & $539 \pm 224$  & $956 \pm 11$  & $883 \pm 52$  & \colorbox{mine}{$965 \pm 5$}  & $904 \pm 41$   & $936 \pm 7$ \\
                    &                             & run    & $233 \pm 136$  & $484 \pm 14$  & $431 \pm 20$  & $486 \pm 22$   & $484 \pm 20$ & \colorbox{mine}{$514 \pm 22$} \\
                    &                             & jump   & $342 \pm 238$  & $722 \pm 4$   & $659 \pm 53$  & $683 \pm 41$  & $679 \pm 72$  & \colorbox{mine}{$808 \pm 16$} \\
\cmidrule(lr){3-9}
                    &                          & whole performance & $330 \pm 165$ & $671 \pm 14$   & $609 \pm 26$   & $684 \pm 18$   & $674 \pm 28$  & \colorbox{mine}{$725 \pm 23$} \\

\midrule
\multirow{16}{*}{APS} & \multirow{6}{*}{Walker}  & walk    & $272 \pm 81$  & $725 \pm 118$  & $705 \pm 130$ & $817 \pm 88$  & $802 \pm 22$  & \colorbox{mine}{$935 \pm 23$} \\
                      &                          & stand   & $663 \pm 54$  & $720 \pm 208$  & $567 \pm 109$  & $751 \pm 65$  & $860 \pm 10$  & \colorbox{mine}{$865 \pm 57$} \\
                      &                          & run     & $139 \pm 25$  & $244 \pm 58$   & $245 \pm 53$  & $316 \pm 26$  & \colorbox{mine}{$332 \pm 26$}  & $267 \pm 14$ \\
                      &                          & flip    & $221 \pm 35$  & $380 \pm 54$  & $431 \pm 78$  & $430 \pm 97$  & \colorbox{mine}{$578 \pm 49$}  & $482 \pm 52$ \\
\cmidrule(lr){3-9}
                      &                          & whole performance & $324 \pm 24$ & $517 \pm 99$ & $487 \pm 75$ & $578 \pm 35$ & \colorbox{mine}{$643 \pm 22$} & $637 \pm 21$ \\

\cmidrule(lr){2-9}
                    & \multirow{6}{*}{Jaco}  & reach top right      & $29 \pm 22$ & $33 \pm 37$ & $18 \pm 10$ & $28 \pm 8$   & $98 \pm 44$   & \colorbox{mine}{$119 \pm 52$} \\
                    &                        & reach top left       & $9 \pm 15$ & $16 \pm 14$ & $11 \pm 11$ & $15 \pm 17$   & $80 \pm 28$   & \colorbox{mine}{$153 \pm 76$} \\
                    &                        & reach bottom right   & $50 \pm 32$ & $33 \pm 27$ & $37 \pm 46$ & $31 \pm 21$   & \colorbox{mine}{$105 \pm 26$}   & $70 \pm 34$ \\
                    &                        & reach bottom left    & $70 \pm 72$ & $6 \pm 8$ & $13 \pm 11$  & $15 \pm 8$   & $52 \pm 20$   & \colorbox{mine}{$186 \pm 26$} \\
\cmidrule(lr){3-9}
                    &                        & whole performance    & $39 \pm 26$ & $22 \pm 14$ & $20 \pm 18$ & $22 \pm 3$    & $84 \pm 16$   & \colorbox{mine}{$132 \pm 16$} \\

\cmidrule(lr){2-9}
                    & \multirow{6}{*}{Quadruped}  & walk    & $312 \pm 158$ & $615 \pm 48$  & $541 \pm 38$  & $601 \pm 90$  & $557 \pm 40$  & \colorbox{mine}{$660 \pm 92$}\\
                    &                             & stand   & $796 \pm 186$ & $898 \pm 57$  & $946 \pm 24$   & $895 \pm 63$  & $934 \pm 9$  & \colorbox{mine}{$950 \pm 12$} \\
                    &                             & run     & $342 \pm 199$ & $471 \pm 32$  & $476 \pm 10$   & $440 \pm 36$  & $466 \pm 13$   & \colorbox{mine}{$472 \pm 27$} \\
                    &                             & jump    & $542 \pm 126$ & $686 \pm 33$  & $692 \pm 6$   & $699 \pm 70$  & \colorbox{mine}{$758 \pm 21$}  & $710 \pm 27$ \\
\cmidrule(lr){3-9}
                    &                             & whole performance   & $498 \pm 160$ & $668 \pm 29$  & $664 \pm 3$   & $659 \pm 50$   & $679 \pm 14$   & \colorbox{mine}{$698 \pm 24$} \\

\midrule
\multirow{16}{*}{PROTO} & \multirow{6}{*}{Walker}   & walk  & $352 \pm 210$ & $889 \pm 45$   & $783 \pm 165$ & $867 \pm 109$ &  $894 \pm 38$ & \colorbox{mine}{$923 \pm 2$} \\
                      &                             & stand & $703 \pm 155$ & \colorbox{mine}{$936 \pm 30$}  & $832 \pm 86$ & $845 \pm 135$ & $931 \pm 22$ & $858 \pm 35$ \\
                      &                             & run   & $202 \pm 69$ & $298 \pm 27$  & $324 \pm 81$ & $262 \pm 69$ & \colorbox{mine}{$387 \pm 29$} & $228 \pm 26$ \\
                      &                             & flip  & $271 \pm 156$ & $477 \pm 47$  & $503 \pm 54$ & $435 \pm 143$ & \colorbox{mine}{$647 \pm 76$} & $645 \pm 62$ \\
\cmidrule(lr){3-9}
                    &                          & whole performance & $382 \pm 129$ & $650 \pm 19$ & $611 \pm 94$ & $602 \pm 112$ & \colorbox{mine}{$715 \pm 31$} & $663 \pm 19$ \\

\cmidrule(lr){2-9}
                    & \multirow{6}{*}{Jaco}  & reach top right      & $5 \pm 9$   & $31 \pm 54$ & $12 \pm 17$   & $37 \pm 52$   & $55 \pm 23$   & \colorbox{mine}{$80 \pm 75$}\\
                    &                        & reach top left       & $2 \pm 2$   & $33 \pm 34$ & $7 \pm 9$     & $14 \pm 13$   & $52 \pm 28$   & \colorbox{mine}{$99 \pm 48$}\\
                    &                        & reach bottom right   & $29 \pm 46$ & $19 \pm 33$ & $25 \pm 36$   & $26 \pm 40$   & $30 \pm 24$   & \colorbox{mine}{$91 \pm 79$}\\
                    &                        & reach bottom left    & $24 \pm 41$ & $0 \pm 0$   & $8 \pm 13$    & $2 \pm 2$     & \colorbox{mine}{$38 \pm 33$}   & $26 \pm 29$\\
\cmidrule(lr){3-9}
                    &                        & whole performance    & $15 \pm 14$ & $21 \pm 26$ & $13 \pm 12$ & $20 \pm 21$   & $44 \pm 19$    & \colorbox{mine}{$74 \pm 26$} \\

\cmidrule(lr){2-9}
                    & \multirow{6}{*}{Quadruped}  & walk    & $131 \pm 32$  & $168 \pm 76$  & $136 \pm 63$  & $235 \pm 144$  & $118 \pm 58$  & \colorbox{mine}{$298 \pm 79$} \\
                    &                             & stand   & $293 \pm 88$  & $351 \pm 263$  & $239 \pm 95$ & $259 \pm 132$  & $329 \pm 95$  & \colorbox{mine}{$569 \pm 55$}\\
                    &                             & run     & $163 \pm 51$  & $125 \pm 112$  & $100 \pm 45$ & $213 \pm 172$  & $175 \pm 44$  & \colorbox{mine}{$283 \pm 33$} \\
                    &                             & jump    & $211 \pm 77$  & $244\pm 92$   & $265 \pm 166$  & $169 \pm 101$  & $242 \pm 70$   & \colorbox{mine}{$407 \pm 7$} \\
\cmidrule(lr){3-9}
                    &                             & whole performance   & $199 \pm 10$ & $222 \pm 107$  & $185 \pm 72$  & $219 \pm 135$  & $216 \pm 54$   & \colorbox{mine}{$389 \pm 44$} \\

\midrule
\multirow{16}{*}{DIAYN} & \multirow{6}{*}{Walker}   & walk  & $174 \pm 151$ & $364 \pm 164$ & $329 \pm 147$ & $348 \pm 174$ & $604 \pm 126$ & \colorbox{mine}{$613 \pm 103$} \\
                      &                             & stand & $558 \pm 129$ & $614 \pm 146$ & $484 \pm 194$ & $429 \pm 83$  & $696 \pm 168$ & \colorbox{mine}{$712 \pm 66$}\\
                      &                             & run   & $96 \pm 25$   & $103 \pm 20$  & $74 \pm 26$   & $75 \pm 25$   & \colorbox{mine}{$156 \pm 25$}  & $152 \pm 38$\\
                      &                             & flip  & $128 \pm 24$  & $272 \pm 58$  & $183 \pm 51$  & $222 \pm 112$ & \colorbox{mine}{$386 \pm 76$}  & $373 \pm 39$\\
\cmidrule(lr){3-9}
                    &                          & whole performance & $239 \pm 79$ & $338 \pm 74$ & $268 \pm 67$ & $268 \pm 97$ & $461 \pm 64$   & \colorbox{mine}{$463 \pm 42$} \\

\cmidrule(lr){2-9}
                    & \multirow{6}{*}{Jaco}  & reach top right      & $43 \pm 37$  & $17 \pm 12$   & $24 \pm 16$   & $1 \pm 1$     & $64 \pm 5$    & \colorbox{mine}{$76 \pm 14$}\\
                    &                        & reach top left       & $10 \pm 12$  & $18 \pm 7$    & $30 \pm 17$   & $55 \pm 2$    & $55 \pm 6$    & \colorbox{mine}{$58 \pm 33$}\\
                    &                        & reach bottom right   & $42 \pm 42$  & $27 \pm 17$   & $7 \pm 4$     & $1 \pm 1$     & $45 \pm 26$   & \colorbox{mine}{$96 \pm 36$}\\
                    &                        & reach bottom left    & $32 \pm 41$  & $28 \pm 12$   & $36 \pm 12$   & $4 \pm 3$     & $46 \pm 4$    & \colorbox{mine}{$82 \pm 35$}\\
\cmidrule(lr){3-9}
                    &                        & whole performance    & $32 \pm 26$ & $22 \pm 6$    & $24 \pm 3$    & $15 \pm 1$    & $52 \pm 7$    & \colorbox{mine}{$78 \pm 11$} \\

\cmidrule(lr){2-9}
                    & \multirow{6}{*}{Quadruped}  & walk    & $196 \pm 151$ & $475 \pm 77$  & $311 \pm 47$  & \colorbox{mine}{$596 \pm 26$}  & $497 \pm 18$  & $482 \pm 11$ \\
                    &                             & stand   & $305 \pm 270$ & $756 \pm 30$  & $793 \pm 87$  & $865 \pm 30$  & \colorbox{mine}{$944 \pm 4$}   & \colorbox{mine}{$944 \pm 9$}\\
                    &                             & run     & $139 \pm 121$ & $380 \pm 22$  & $364 \pm 22$  & $408 \pm 15$  & $469 \pm 2$   & \colorbox{mine}{$472 \pm 5$}\\
                    &                             & jump    & $190 \pm 158$ & $636 \pm 19$  & $576 \pm 47$  & $700 \pm 18$  & \colorbox{mine}{$769 \pm 8$}   & $766 \pm 25$\\
\cmidrule(lr){3-9}
                    &                          & whole performance  & $207 \pm 168$ & $562 \pm 23$  & $511 \pm 37$  & $643 \pm 14$ & \colorbox{mine}{$670 \pm 4$}    & $666 \pm 2$ \\

\bottomrule
\end{tabular}}
\label{tab:fullresult}
\end{table}

\newpage
\begin{table}[H]
\vspace{-15pt}
\centering
\caption{\small Small sample dataset experimental results on ExORL.}
\resizebox{0.71\linewidth}{!}{\scriptsize
\begin{tabular}{lcccccc}
\toprule
\textbf{Dataset} & \textbf{Domain} & \textbf{Task} & \textbf{FB} & \textbf{VCFB} & \textbf{MCFB}  & \textbf{BREEZE (ours)}\\ 
\midrule
\multirow{16}{*}{RND} & \multirow{6}{*}{Walker}  & walk    & $300 \pm 50$ & $377 \pm 18$ & $225 \pm 46$ & \colorbox{mine}{$663 \pm 41$} \\
                      &                          & stand   & $368 \pm 70$ & $556 \pm 44$ & $587 \pm 97$ & \colorbox{mine}{$791 \pm 21$} \\
                      &                          & run     & $104 \pm 11$ & $144 \pm 17$ & $101 \pm 18$ & \colorbox{mine}{$224 \pm 12$} \\
                      &                          & flip    & $284 \pm 91$ & $319 \pm 46$ & $236 \pm 44$ & \colorbox{mine}{$421 \pm 9$} \\
\cmidrule(lr){3-7}
                      &                          & whole performance & $264 \pm 33$ & $350 \pm 29$ & $287 \pm 48$ & \colorbox{mine}{$525 \pm 13$} \\
\cmidrule(lr){2-7}
                    & \multirow{6}{*}{Jaco}  & reach top right     & $9 \pm 8$      & $6 \pm 7$     & $28 \pm 16$   & \colorbox{mine}{$51 \pm 10$} \\
                    &                        & reach top left      & $2 \pm 2$      & $6 \pm 5$     & $1 \pm 1$     & \colorbox{mine}{$7 \pm 3$} \\
                    &                        & reach bottom right  & $10 \pm 10$    & $4 \pm 4$     & $19 \pm 15$   & \colorbox{mine}{$38 \pm 22$}\\
                    &                        & reach bottom left   & $8 \pm 9$      & $20 \pm 12$   & $1 \pm 1$     & \colorbox{mine}{$47 \pm 22$} \\
\cmidrule(lr){3-7}
                    &                        & whole performance   & $7 \pm 5$ & $9 \pm 2$ & $13 \pm 7$ & \colorbox{mine}{$36 \pm 5$} \\
                    
\cmidrule(lr){2-7}
                    & \multirow{6}{*}{Quadruped}  & walk   & $125 \pm 98$   & $164 \pm 45$ & $94 \pm 58$    & \colorbox{mine}{$382 \pm 15$} \\
                    &                             & stand  & $268 \pm 183$  & $288 \pm 73$ & $159 \pm 70$   & \colorbox{mine}{$651 \pm 37$} \\
                    &                             & run    & $132 \pm 83$   & $186 \pm 52$ & $89 \pm 49$    & \colorbox{mine}{$318 \pm 16$} \\
                    &                             & jump   & $178 \pm 130$  & $292 \pm 82$ & $148 \pm 70$   & \colorbox{mine}{$544 \pm 37$} \\
\cmidrule(lr){3-7}
                    &                          & whole performance & $176 \pm 123$ & $233 \pm 52$ & $123 \pm 61$ & \colorbox{mine}{$474 \pm 21$} \\

\midrule
\multirow{16}{*}{APS} & \multirow{6}{*}{Walker}  & walk    & $535 \pm 139$  & $533 \pm 114$ & $449 \pm 64$ & \colorbox{mine}{$762 \pm 25$}\\
                      &                          & stand   & $447 \pm 45$   & $625 \pm 52$  & $562 \pm 104$ & \colorbox{mine}{$828 \pm 19$}\\
                      &                          & run     & $166 \pm 73$   & \colorbox{mine}{$216 \pm 33$}  & $192 \pm 50$ & $200 \pm 2$ \\
                      &                          & flip    & $334 \pm 121$  & $293 \pm 85$  & $354 \pm 95$ & \colorbox{mine}{$365 \pm 24$} \\
\cmidrule(lr){3-7}
                      &                          & whole performance & $370 \pm 66$ & $416 \pm 10$ & $389 \pm 77$ & \colorbox{mine}{$539 \pm 15$}\\

\cmidrule(lr){2-7}
                    & \multirow{6}{*}{Jaco}  & reach top right      & \colorbox{mine}{$36 \pm 29$}   & $7 \pm 7$     & $20 \pm 24$   & $21 \pm 16$ \\
                    &                        & reach top left       & $3 \pm 2$     & $7 \pm 5$     & \colorbox{mine}{$24 \pm 19$}   & $21 \pm 15$ \\
                    &                        & reach bottom right   & $44 \pm 60$   & $42 \pm 43$   & $56 \pm 71$   & \colorbox{mine}{$58 \pm 33$}\\
                    &                        & reach bottom left    & $2 \pm 3$     & $0 \pm 0$     & $15 \pm 4$    & \colorbox{mine}{$51 \pm 29$}\\
\cmidrule(lr){3-7}
                    &                        & whole performance    & $21 \pm 17$   & $14 \pm 13$ & $29 \pm 27$ & \colorbox{mine}{$38 \pm 9$} \\

\cmidrule(lr){2-7}
                    & \multirow{6}{*}{Quadruped}  & walk    & $299 \pm 27$  & $293 \pm 26$  & $255 \pm 116$ & \colorbox{mine}{$435 \pm 49$} \\
                    &                             & stand   & $479 \pm 12$  & $524 \pm 141$ & $456 \pm 92$  & \colorbox{mine}{$810 \pm 85$} \\
                    &                             & run     & $260 \pm 10$  & $232 \pm 46$  & $222 \pm 101$ & \colorbox{mine}{$401 \pm 30$}\\
                    &                             & jump    & $321 \pm 115$ & $356 \pm 30$  & $338 \pm 179$ & \colorbox{mine}{$577 \pm 59$} \\
\cmidrule(lr){3-7}
                    &                             & whole performance   & $340 \pm 29$ & $351 \pm 57$ & $318 \pm 122$ & \colorbox{mine}{$556 \pm 52$} \\

\midrule
\multirow{16}{*}{PROTO} & \multirow{6}{*}{Walker}   & walk  & $522 \pm 136$ & \colorbox{mine}{$732 \pm 38$} & $685 \pm 87$ & $709 \pm 27$ \\
                      &                             & stand & $608 \pm 119$ & $709 \pm 36$ & $637 \pm 39$ & \colorbox{mine}{$846 \pm 22$} \\
                      &                             & run   & $176 \pm 58$  & \colorbox{mine}{$215 \pm 24$} & $174 \pm 49$ & $210 \pm 7$ \\
                      &                             & flip  & $354 \pm 67$  & $394 \pm 36$ & $357 \pm 43$ & \colorbox{mine}{$447 \pm 27$} \\
\cmidrule(lr){3-7}
                    &                          & whole performance & $415 \pm 19$ & $513 \pm 31$ & $463 \pm 11$ & \colorbox{mine}{$553 \pm 18$} \\

\cmidrule(lr){2-7}
                    & \multirow{6}{*}{Jaco}  & reach top right      & $7 \pm 10$ & $17 \pm 19$ & $6 \pm 8$ & \colorbox{mine}{$18 \pm 13$}\\
                    &                        & reach top left       & \colorbox{mine}{$59 \pm 4$} & $14 \pm 12$ & $19 \pm 14$ & $51 \pm 22$ \\
                    &                        & reach bottom right   & $0 \pm 0$ & $22 \pm 22$ & $12 \pm 10$ & \colorbox{mine}{$27 \pm 19$} \\
                    &                        & reach bottom left    & $0 \pm 0$ & $18 \pm 11$ & $12 \pm 10$ & \colorbox{mine}{$19 \pm 12$} \\
\cmidrule(lr){3-7}
                    &                        & whole performance    & $16 \pm 2$ & $18 \pm 12$ & $12 \pm 7$ & \colorbox{mine}{$29 \pm 12$} \\

\cmidrule(lr){2-7}
                    & \multirow{6}{*}{Quadruped}  & walk    & $105 \pm 63$  & $22 \pm 7$    & $103 \pm 85$  & \colorbox{mine}{$129 \pm 66$} \\
                    &                             & stand   & $327 \pm 180$ & $196 \pm 202$ & \colorbox{mine}{$387 \pm 185$} & $59 \pm 35$ \\
                    &                             & run     & $169 \pm 99$  & $98 \pm 110$  & $206 \pm 152$ & \colorbox{mine}{$244 \pm 173$} \\
                    &                             & jump    & $193 \pm 111$ & $106 \pm 103$ & $265 \pm 194$ & \colorbox{mine}{$291 \pm 47$} \\
\cmidrule(lr){3-7}
                    &                             & whole performance   & $198 \pm 111$ & $106 \pm 103$ & \colorbox{mine}{$240 \pm 134$} & $181 \pm 60$ \\

\midrule
\multirow{16}{*}{DIAYN} & \multirow{6}{*}{Walker}   & walk  & $103 \pm 79$  & $114 \pm 16$  & $129 \pm 38$  & \colorbox{mine}{$418 \pm 119$} \\
                      &                             & stand & $513 \pm 245$ & $460 \pm 256$ & $424 \pm 27$  & \colorbox{mine}{$524 \pm 39$} \\
                      &                             & run   & $87 \pm 45$   & $79 \pm 22$   & $63 \pm 19$   & \colorbox{mine}{$102 \pm 17$} \\
                      &                             & flip  & $106 \pm 50$  & $186 \pm 44$  & $169 \pm 33$  & \colorbox{mine}{$276 \pm 28$} \\
\cmidrule(lr){3-7}
                    &                          & whole performance & $202 \pm 94$ & $210 \pm 81$ & $196 \pm 26$ & \colorbox{mine}{$330 \pm 43$} \\

\cmidrule(lr){2-7}
                    & \multirow{6}{*}{Jaco}  & reach top right      & $18 \pm 23$   & $14 \pm 13$   & \colorbox{mine}{$29 \pm 37$}   & $5 \pm 34$ \\
                    &                        & reach top left       & $32 \pm 35$   & $25 \pm 9$    & $9 \pm 8$     & \colorbox{mine}{$72 \pm 58$} \\
                    &                        & reach bottom right   & $15 \pm 11$   & $12 \pm 9$    & \colorbox{mine}{$42 \pm 58$} & $6 \pm 29$ \\
                    &                        & reach bottom left    & $5 \pm 2$     & \colorbox{mine}{$19 \pm 24$}   & $0 \pm 0$ & $7 \pm 26$ \\
\cmidrule(lr){3-7}
                    &                        & whole performance    & $17 \pm 11$   & $18 \pm 5$ & $20 \pm 26$  & \colorbox{mine}{$22 \pm 15$} \\

\cmidrule(lr){2-7}
                    & \multirow{6}{*}{Quadruped}  & walk    & $158 \pm 41$ & $189 \pm 73$ & $190 \pm 10$ & \colorbox{mine}{$303 \pm 69$} \\
                    &                             & stand   & $430 \pm 84$ & $423 \pm 46$ & $408 \pm 57$ & \colorbox{mine}{$680 \pm 117$} \\
                    &                             & run     & $224 \pm 41$ & $214 \pm 28$ & $204 \pm 28$ & \colorbox{mine}{$341 \pm 79$} \\
                    &                             & jump    & $368 \pm 83$ & $327 \pm 58$ & $343 \pm 63$ & \colorbox{mine}{$459 \pm 61$} \\
\cmidrule(lr){3-7}
                    &                          & whole performance  & $295 \pm 46$ & $288 \pm 48$ & $286 \pm 34$ & \colorbox{mine}{$446 \pm 78$} \\

\bottomrule
\end{tabular}}
\label{tab:smallsampleresult}
\end{table}

\newpage
\subsection{Empirically Evaluated $M^{\pi_z}$ and $Q_z$ on Other Tasks}
\label{app:empirical_motivation}

We present empirical visualization results of $M^{\pi_z}$ and $Q_z$ distributions, shown in Figure~\ref{fig:dis1}-\ref{fig:dis3}. We trained FB, VCFB, and MCFB on the RND~\citep{Burda2018ExplorationBR} dataset. Specifically, the trained $B$ network was utilized for inferring the task latent vector $z_\text{task}$, and as the input to both $M^{\pi_z}$ and $Q_z$ during the evaluation phase.

\begin{figure}[H]
\centering
\begin{minipage}{\linewidth}
\centering
\includegraphics[width=\linewidth]{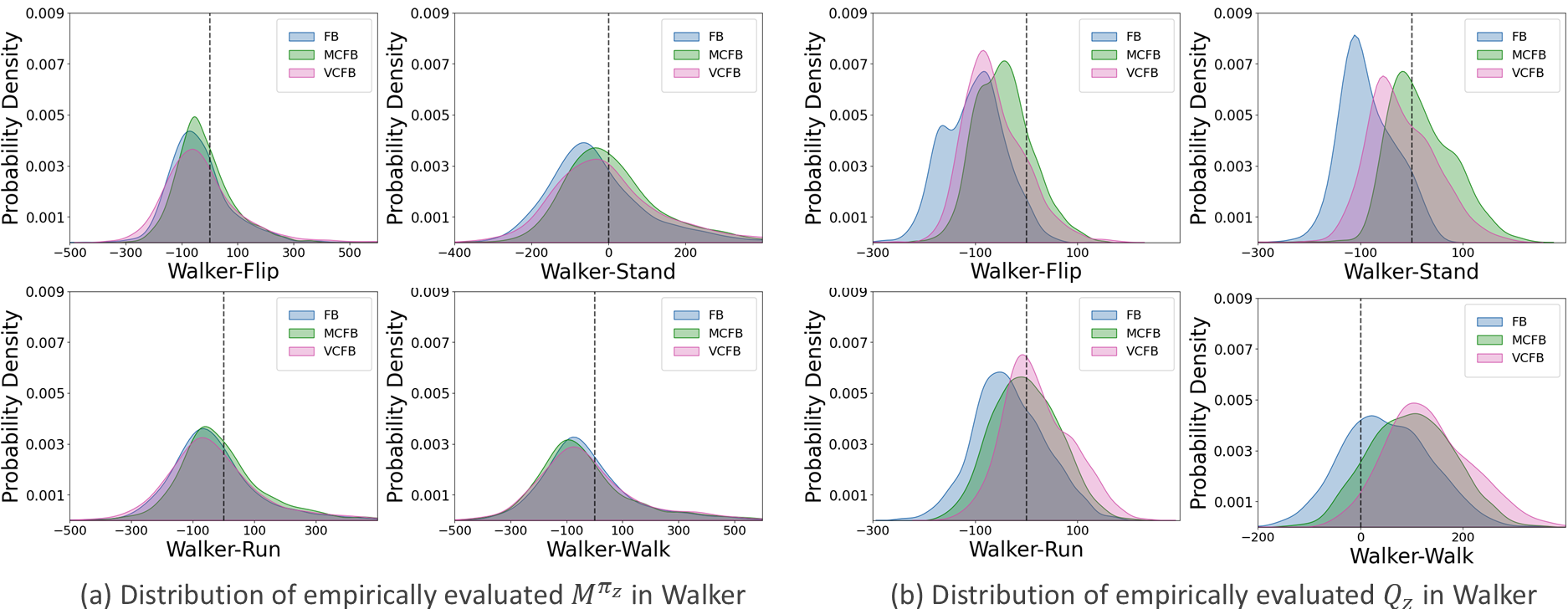}
\captionsetup{skip=3pt}
\caption{\small Visualization of the empirical $M^{\pi_z}$ and $Q_z$ distribution on Walker RND.} 
\label{fig:dis1}
\vspace{20pt}
\end{minipage}
% \hfill
\begin{minipage}{\linewidth}
\centering
\includegraphics[width=\linewidth]{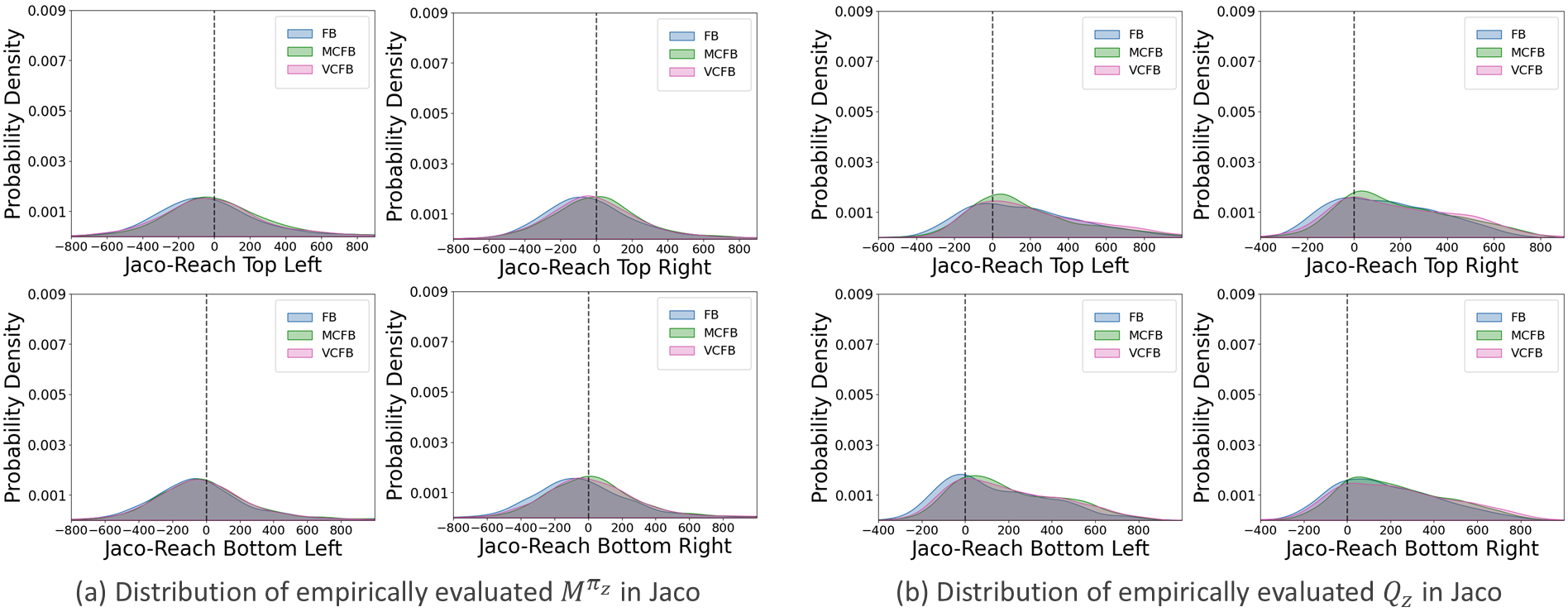}
\captionsetup{skip=3pt}
\caption{\small Visualization of the empirical $M^{\pi_z}$ and $Q_z$ distribution on Jaco RND.} 
\label{fig:dis2}
\vspace{20pt}
\end{minipage}
% \hfill
\begin{minipage}{\linewidth}
\centering
\includegraphics[width=\linewidth]{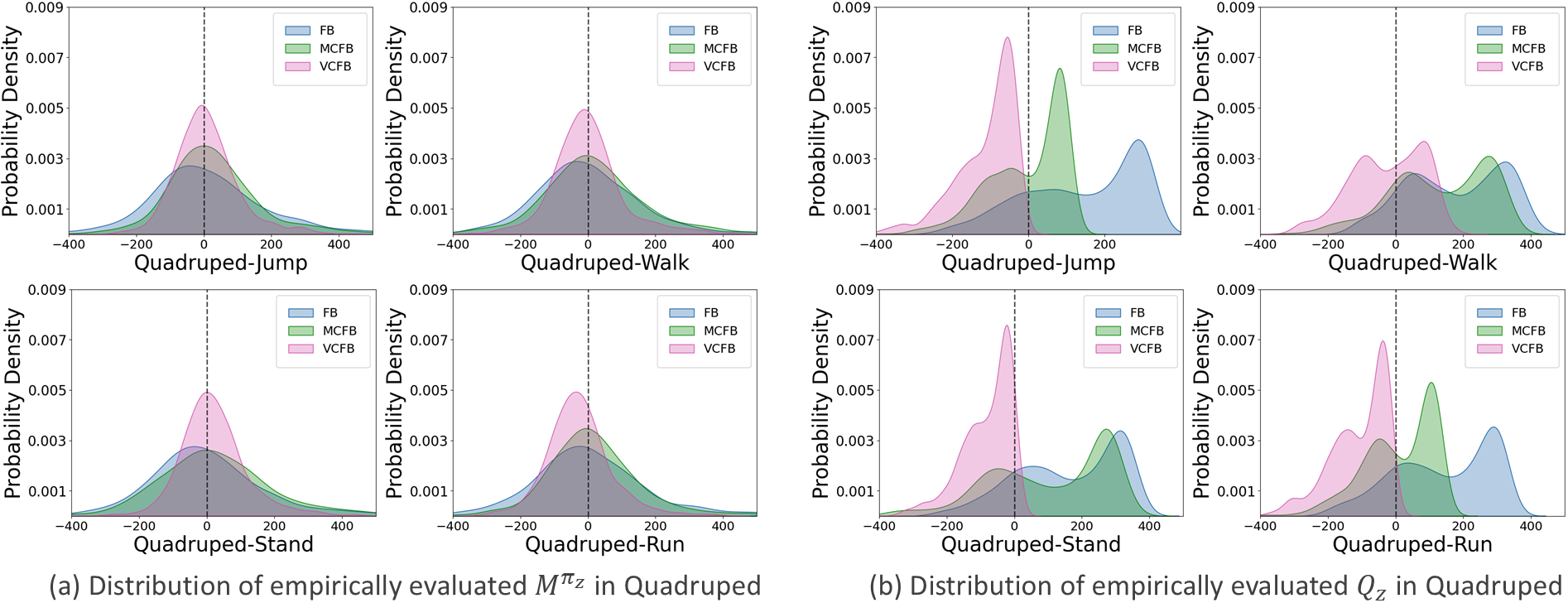}
\captionsetup{skip=3pt}
\caption{\small Visualization of the empirical $M^{\pi_z}$ and $Q_z$ distribution on Quadruped RND.} 
\label{fig:dis3}
\end{minipage}
\end{figure}
\section{Discussion and Limitation}
\label{appendix:limitation}

\textit{\textbf{E1}. Hyperparameter Sensitivity and Tuning.}

BREEZE introduces additional components compared to vanilla FB, resulting in more hyperparameters. While performance is generally robust, we provide the following tuning guidance:

\begin{itemize}[leftmargin=*]
    \item \textit{Effect of components.}
    
    Representation architectures are responsible for a high performance ceiling, as we can observe from Table~\ref{table:component_abl}. The diffusion policy ensures a solid performance lower bound by maintaining stable alignment with any quality dataset.
    
    \item \textit{Greedy demand for upstream value learning.} 
    
    Training cost is mainly influenced by the action ratio $\rho_a$ and the number of rejection sampling candidates $K_\text{train}$. While the diffusion policy provides conservatism, we find that increasing $\rho_a$ and $K_\text{train}$ further improves performance in some domains at the cost of additional computation. Near-optimal results can be achieved cost-effectively by reducing the batch size and diffusion timesteps to reallocate the computational resources to increase $\rho_a$ and $K_\text{train}$.
\end{itemize}

\textit{\textbf{E2.} Computation and Performance Trade-off.}

A limitation of BREEZE is its higher computational cost, attributable to the expressive representation networks and diffusion model. We consider this a reasonable trade-off given the significant improvements in stability and zero-shot performance. For a comparable and fair comparison, all reported results use a batch size of 512, and training for 1 million steps takes approximately 20 hours or less for environments with standard TD updates~\citep{suttonTD}; configurations with higher $K_\text{train}$ and mixture ratios require up to 39 hours. As shown in Table~\ref{tab:computation}, BREEZE typically converges within 400k steps, matching or exceeding baselines trained for 1M steps. Future work may explore framework-level optimizations, such as replacing diffusion with more efficient flow-based policies to lower costs without compromising performance.

\begin{table}[H]
\centering
\caption{\small \textbf{Computation cost v.s. performance trade-off.} Aggregated scores are averaged over 5 random seeds across datasets. Experiments are conducted on a single NVIDIA A6000 GPU.}
\resizebox{0.99\linewidth}{!}{\scriptsize
\begin{tabular}{lcccccc}
\toprule
\textbf{Methods} & \textbf{Training Steps (k)} & \textbf{Training Time (h)} & \textbf{IQM Walker} & \textbf{IQM Jaco} & \textbf{IQM Quadruped} & \textbf{Aggregate IQM}\\ 
\midrule
\multirow{3}{*}{FB}   & \cellcolor{gray!20}200  & \cellcolor{gray!20}0.8 & \cellcolor{gray!20}480 & \cellcolor{gray!20}16 & \cellcolor{gray!20}303 & \cellcolor{gray!20}266 \\
                      & 400  & 1.6 & 508 & 21 & 401 & 310 \\
                      & 1000  & 4.0 & 542 & 24 & 531 & 366 \\
\cmidrule(lr){1-7}
\multirow{3}{*}{VCFB}   & \cellcolor{gray!20}200  & \cellcolor{gray!20}2.4 & \cellcolor{gray!20}441 & \cellcolor{gray!20}18 & \cellcolor{gray!20}361 & \cellcolor{gray!20}273\\
                      & 400  & 4.8 & 489 & 19 & 457 &322\\
                      & 1000  & 12.0 & 505 & 26 & 492 & 341 \\ 
\cmidrule(lr){1-7}
\multirow{3}{*}{MCFB}   & \cellcolor{gray!20}200  & \cellcolor{gray!20}2.4 & \cellcolor{gray!20}473 & \cellcolor{gray!20}21 & \cellcolor{gray!20}347 & \cellcolor{gray!20}280\\ 
                      & 400  & 4.8 & 514 & 21 & 469 & 335 \\
                      & 1000  & 12.0 & 527 & 25 & 551 & 368 \\ 
\cmidrule(lr){1-7}
\multirow{2}{*}{BREEZE (ours)}   & \cellcolor{gray!20}200  & \cellcolor{gray!20}4.0-7.8 & \cellcolor{gray!20}586 & \cellcolor{gray!20}57 & \cellcolor{gray!20}550 & \cellcolor{gray!20}398 \\
                    & 400  & 8.0-15.6 & \textbf{606} & \textbf{75} & \textbf{567} & \textbf{416} \\
\bottomrule
\end{tabular}}
\label{tab:computation}
\end{table}

\end{document}